\pdfoutput=1
\documentclass[twoside]{article}

\usepackage[accepted]{aistats2025}
\usepackage{hyperref}
\usepackage{url}
\usepackage{amsthm}
\newtheorem{definition}{Definition}[section]\usepackage{lineno}
\usepackage{tikz}
\usetikzlibrary{positioning}
\usetikzlibrary{arrows}
\usepackage{caption}
\usepackage{subcaption}
\usepackage{titlesec}
\usepackage{graphicx} 
\usepackage{float}
\usepackage{booktabs}
\usepackage{enumitem}
\usepackage{booktabs}
\usepackage{multirow}
\usepackage{subcaption}
\usepackage{hyperref}
\usepackage{url}
\usepackage{algorithm}
\usepackage{algpseudocode}
\usepackage{amsmath}
\usepackage{nicefrac}
\usepackage{amssymb}
\usepackage{bbm}
\usepackage{amsthm}
\usepackage{mathtools}
\usepackage{cleveref}
\usepackage{cancel}
\theoremstyle{plain}
\usepackage{xfrac}

\def\eqref#1{equation~\ref{#1}}

\def\1{\bm{1}}

\DeclareMathAlphabet{\mathsfit}{\encodingdefault}{\sfdefault}{m}{sl}
\SetMathAlphabet{\mathsfit}{bold}{\encodingdefault}{\sfdefault}{bx}{n}

\newtheorem{theorem}{Theorem}[section]

\newtheorem{proposition}[theorem]{Proposition}

\DeclarePairedDelimiterX{\infdivent}[1]{[}{]}{%
  #1%
}
\DeclarePairedDelimiterX{\infdivx}[2]{[}{]}{%
  #1\;\delimsize\|\;#2%
}
\DeclarePairedDelimiterX{\infdivxsemicolon}[2]{[}{]}{%
  #1\delimsize;#2%
}

\usepackage{tocbasic}
\DeclareTOCStyleEntry[
  beforeskip=0em,%
  pagenumberformat=\textbf,
]{tocline}{section}

\newtheorem{manualpropinner}{Proposition}
\newenvironment{manualprop}[1]{%
  \renewcommand\themanualpropinner{#1}%
  \manualpropinner
}{\endmanualpropinner}

\algrenewcommand\algorithmicrequire{\textbf{Input:}}
\algrenewcommand\algorithmicensure{\textbf{Output:}}

\let\underbrace\LaTeXunderbrace

\usepackage[round]{natbib}

\begin{document}

\runningauthor{He$^*$, Chen$^*$, Zhang$^*$, Barber, Hernández-Lobato}

\twocolumn[

\aistatstitle{Training Neural Samplers with Reverse Diffusive KL Divergence}

\aistatsauthor{Jiajun He$^{*,1}$  \And Wenlin Chen$^{*,1,2}$ \And  Mingtian Zhang$^{*,3}$ 
} \vspace{6pt}
\aistatsauthor{
David Barber$^{3}$ \And  José Miguel Hernández-Lobato$^{1}$ \vspace{8pt}}
\aistatsaddress{$^*$Equal contribution 
 }\vspace{-15pt}
\aistatsaddress{$^1$University of Cambridge \And  $^2$MPI for Intelligent Systems \And $^3$University College London}\vspace{-15pt}
\aistatsaddress{ \texttt{jh2383@cam.ac.uk\ \ \ \ \ \  wc337@cam.ac.uk\ \ \ \ \ \ m.zhang@cs.ucl.ac.uk}}]

\begin{abstract}
Training generative models to sample from unnormalized density functions is an important and challenging task in machine learning. 
Traditional training methods often rely on the reverse Kullback-Leibler (KL) divergence due to its tractability. 
However, the mode-seeking behavior of reverse KL hinders effective approximation of multi-modal target distributions. 
To address this, we propose to minimize the reverse KL along diffusion trajectories of both model and target densities. 
We refer to this objective as the \emph{reverse diffusive KL divergence}, which allows the model to capture multiple modes. 
Leveraging this objective, we train neural samplers that can efficiently generate samples from the target distribution in \emph{one step}.
We demonstrate that our method enhances sampling performance across various Boltzmann distributions, including both synthetic multi-modal densities and $n$-body particle systems.

\end{abstract}

\section{INTRODUCTION}
Sampling from unnormalized distributions is an essential and challenging research problem with wide applications in machine learning, Bayesian inference, and statistical mechanics.
Consider a target distribution with an analytical but unnormalized density function:
\begin{align}
     p_d(x)=\exp(-E(x))/Z,
\end{align}
where $x$ is the random variable to be sampled, $E:\mathbb{R}^n\to\mathbb{R}$ is a lower-bounded differentiable energy function, and $Z=\int \exp(-E(x))dx$ is the intractable normalization constant. 
A common approach to sampling from $p_d(x)$ involves designing MCMC samplers~\citep{neal2011mcmc, chen2024diffusive}. 
However, for high-dimensional, multi-modal target distributions, MCMC methods often take a long time to converge and need to simulate a very long chain for the samples to be uncorrelated~\citep{pompe2020framework}. 
This presents significant challenges in large-scale simulation problems.

Alternatively, one can approximate the target distribution $p_d(x)$ with a generative model $p_\theta(x)$, such as a normalizing flow or a latent variable model $p_\theta(x)\coloneqq\int p_\theta(x|z)p(z)dz$, which is easier to sample from.
This model is often referred to as a \emph{neural sampler} \citep{levy2017generalizing,wu2020stochastic, arbel2021annealed,di2021neural}.
Training a neural sampler involves learning the model parameters $\theta$, which is usually achieved by minimizing a divergence between $p_\theta(x)$ and $p_d(x)$.
A common choice to train the neural sampler is the reverse KL divergence due to its tractability.

However, the most significant limitation of reverse KL is the \emph{mode collapse} phenomenon due to its mode-seeking behavior~\citep{bishop2006pattern}. 
This means that when the target distribution $p_d(x)$ contains multiple distant modes, the model $p_{\theta}(x)$ trained by reverse KL will underestimate the variance of $p_d(x)$ and can only capture a few modes. 
This is undesirable since the target distributions, such as the Bayesian posteriors~\citep{welling2011bayesian} and Boltzmann distributions~\citep{noe2019boltzmann}, often exhibit multiple modes. 
\par
In this paper, we propose to use an alternative objective, the \emph{diffusive KL divergence} (DiKL). 
This objective convolves both the target and the model distributions with Gaussian diffusion kernels, allowing for better connectivity and merging of distant modes in the noisy space. 
Notably, DiKL is still a valid divergence between the model density and the \emph{original} target distribution, which allows us to learn the original target with better mass-covering capability.
We further introduce a tractable gradient estimator for reverse DiKL, enabling practical training of neural samplers with this divergence.
We demonstrate the effectiveness of our approach on both synthetic and $n$-body system targets, where it matches previous state-of-the-art models with reduced training and sampling costs.

\section{BACKGROUND: KL DIVERGENCE}\label{background:rkl}
Given access to target samples $ \{x_1, \dots, x_N\} \sim p_d(x)$, it is common to minimize the forward KL divergence to fit a generative model $p_\theta(x)$ to the target $p_d(x)$:
\begin{align}
    \mathrm{KL}(p_d || p_\theta) \approx -\frac{1}{N} \sum_{n=1}^N \log p_\theta(x_n)+\text{const.},
\end{align}
which is equivalent to maximum likelihood estimation (MLE).
However, our setting only assumes access to the unnormalized density of $p_d(x)$ without samples, where the reverse KL divergence is typically employed:
\begin{align}
    &\mathrm{KL}(p_\theta||p_d)=\int (\log p_\theta(x)-\log p_d(x) )p_\theta(x)dx\nonumber\\
    &\qquad =\int (\log p_\theta(x)+E(x))p_\theta(x)dx+\log Z,\label{eq:rkl}
\end{align}
where $\log Z$ is a constant independent of $x$. 
The integration over $p_\theta(x)$ can be approximated by the Monte Carlo method with samples from $p_\theta(x)$, and the gradient of the reverse KL w.r.t. the model parameter $\theta$ can be obtained by auto differentiation with the reparameterization trick~\citep{kingma2013auto}. 
This objective is particularly suitable for models with analytically tractable marginal densities, such as normalizing flows~\citep{papamakarios2019normalizing, rezende2020normalizing,dinh2016density,kingma2018glow}.

For other models like a latent variable model
$p_\theta(x)=\int p_\theta(x|z)p(z)dz$, the log marginal $\log p_\theta(x)$ is typically intractable.
Instead, one can derive a tractable upper bound of the reverse KL~\citep{zhang2019variational}: 
\begin{align}
    \mathrm{KL}(p_\theta||p_d)\leq \mathrm{KL}(p_\theta(x|z)p(z)||q_\phi(z|x)p_d(x)),\label{eq:upper_bound_rkl}
\end{align}
where $q_\phi(z|x)$ is a learnable variational distribution. This reverse KL upper bound contrasts with the more commonly studied forward KL upper bound, as discussed in~\citet{wainwright2008graphical,kingma2013auto}.
While this variational approach circumvents the intractability of $\log p_{\theta}(x)$, it introduces its own challenges and limitations, such as the limited flexibility of the variational family and the potential looseness of the variational bound.

Alternatively, \citet{li2017gradient,shi2017kernel,song2020sliced,luo2024entropy}~directly derive the analytical form of the gradient of the negative entropy for the reverse KL divergence in \Cref{eq:rkl}:
\begin{multline}
    \nabla_\theta \int \log p_\theta(x)p_\theta(x)dx =\int p_\theta(x) \nabla_x\log p_\theta(x)\frac{\partial x}{\partial\theta}dx,\label{eq:neg_entropy_grad}
\end{multline}
where the score function $\nabla_x\log p_\theta(x)$ of the model can be approximated by training a score network using the model samples with score matching~\citep{hyvarinen2005estimation}. This enables fitting latent variable models to unnormalized target densities by reverse KL minimization without any variational approximation. However, the mode-seeking behavior of reverse KL typically results in the trained model $p_{\theta}(x)$ collapsing to a small number of modes in a multi-modal target distribution.

\section{DIFFUSIVE KL DIVERGENCE}
One effective way to bridge and merge isolated modes is Gaussian convolution, which has been successfully used in training diffusion models~\citep{sohl2015deep,song2021score,ho2020denoising} and encouraging exploration of samplers~\citep{lee2021structured,RDMC,chen2024diffusive}.
It can also potentially reduce the number of modes due to the fact that Gaussian convolution effectively convexifies any functions to its convex envelope~\citep{mobahi2015link}. In \Cref{fig:conv}, we provide a toy visualization showing that by increasing the variance of the Gaussian convolution, we can bridge modes or even reduce the number of modes in the original multi-modal distribution.

\begin{figure}
    \centering
\includegraphics[width=\linewidth]{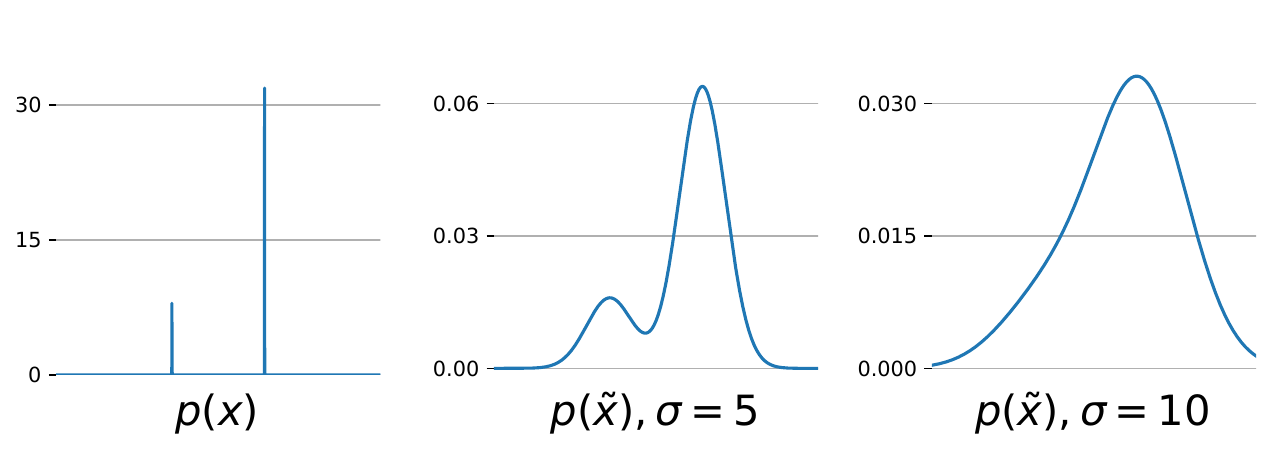}
    \caption{We convolve a Gaussian kernel $\mathcal{N}(\tilde{x}|x,\sigma^2)$ with $\sigma\in\{5,10\}$ to the original distribution $p(x)$. This demonstrates that Gaussian convolution can bridge modes and even reduce the number of modes as the variance of the Gaussian increases.}
    \label{fig:conv}
\end{figure}

To construct a valid divergence that leverages Gaussian convolutions, one can convolve two distributions $p(x)$ and $q(x)$ with the same Gaussian kernel $k(\tilde{x}|x)=\mathcal{N}(\tilde{x}|\alpha x, \sigma^2I)$ and then define the  KL divergence between the convolved distributions $\tilde{p}(\tilde{x})=\int k(\tilde{x}|x)p(x)dx$ and $\tilde{q}(\tilde{x})=\int k(\tilde{x}|x)q(x)dx$. This type of divergence construction is known as the spread divergence~\citep{zhang2020spread}. 
\begin{definition}[Spread KL Divergence]
\begin{align}
    &\mathrm{SKL}_k(p||q)\equiv \mathrm{KL}( \tilde{p}|| \tilde{q})=\mathrm{KL}( p*k|| q*k),
\end{align}
where $*$ denotes the convolution operator:
$\tilde{\pi}\equiv\pi*k\equiv \int k(\tilde{x}|x) \pi(x) dx$.
\end{definition}
The spread KL divergence is a theoretically well-defined divergence as $\mathrm{SKL}_k(p||q)=0\Leftrightarrow p=q$ for any Gaussian kernel $k$, as shown in~\citet{zhang2020spread}. In practice, the choice of 
$k$ is crucial for model training, and selecting the optimal kernel is a challenging problem. Inspired by the recent success of diffusion models, instead of selecting one $k$, one can use a sequence of Gaussian kernels with different lengthscales to construct a ``multi-level spread KL divergence'', which we refer to as \emph{diffusive KL divergence} (DiKL).
\begin{definition}[Diffusive KL Divergence]
\begin{align}
    \mathrm{DiKL}_{\mathcal{K}}(p||q) \equiv \sum_{t=1}^T w(t)\mathrm{KL}(p*k_t||q*k_t),
\end{align}
where $w(t)$ is a positive scalar weighting function and $\mathcal{K}=\{k_1,\cdots, k_T\}$ is a set of (scaled) Gaussian convolution kernels denoted as $k_t(x_t|x)=\mathcal{N}(x_t|\alpha_t x, \sigma_t^2I)$. 
\end{definition}
Since the DiKL can be seen as an average of multiple spread KL divergence with different Gaussian kernels, it is straightforward to show that it is a valid divergence, i.e., $\mathrm{DiKL}_{\mathcal{K}}(p||q)=0\Leftrightarrow p=q$. 

The DiKL has been successfully applied to some important applications such as 3D generative models~\citep{pooledreamfusion,wang2024prolificdreamer} and diffusion distillation~\citep{luo2024diff, xie2024distillation}. However, in these cases, $p * k_t$ corresponds to a given pre-trained diffusion model. In contrast, our setting only assume access to the unnormalized target density without any samples, and therefore $p*k_t$ is usually intractable. In the next section, we propose a practical gradient estimator of the reverse diffusion KL divergence for training neural samplers to capture diverse modes in unnormalized target densities.

\section{TRAINING NEURAL SAMPLERS WITH DIKL}
We focus on training neural samplers defined by a latent variable model:
\begin{align}
    p_\theta(x) = \int p_\theta(x|z) p(z) dz,\label{eq:lvm}
\end{align}
where $p(z) = \mathcal{N}(z|0,I)$ and $p_\theta(x|z)$ is parameterized by a neural network:  $p_\theta(x|z)= p(x| g_\theta(z))$.

Unlike the conventional KL divergence, which requires the model $p_{\theta}(x)$ to have a valid density function~\citep{arjovsky2017wasserstein}, SKL and DiKL are well-defined even for singular distributions (e.g., delta function), as discussed by \citet{zhang2020spread}. Therefore, we can let the generator $p_\theta(x|z)$ be a deterministic function $g_\theta$ and define the ``generalized model density''\footnote{In this case, the marginal distribution may not be absolutely continuous (a.c.) w.r.t. the Lebesgue measure, which implies that it may not have a valid density function, e.g., when $\text{Dim}(z) < \text{Dim}(x)$.} as:
\begin{align}
    p_\theta(x) = \int \delta(x - g_\theta(z)) p(z) dz.
\end{align}
This type of model is also referred to as an implicit model~\citep{goodfellow2014generative,huszar2017variational}, which is more flexible than the latent variable model defined in \Cref{eq:lvm}, as it avoids pre-defining a constrained distribution family for $p(x| g_\theta(z))$.

We now explore how to train such a neural sampler $p_\theta(x)$ to fit the unnormalized target density $p_d(x)$ using the reverse DiKL, denoted as $\mathrm{DiKL}_{\mathcal{K}}(p_\theta || p_d)$. For simplicity, we first consider DiKL with a single kernel $k_t$; the extension to multiple kernels is straightforward. The reverse DiKL can be expressed as:
\begin{align}
    &\mathrm{DiKL}_{k_t}(p_\theta || p_d) \equiv \mathrm{KL}(p_\theta * k_t || p_d * k_t) \nonumber \\
    =& \int p_\theta(x_t) \left( \log p_\theta(x_t) - \log p_d(x_t) \right) dx_t, \label{eq:dikl:single}
\end{align}
where $p_\theta(x_t) = \int k_t(x_t | x) p_\theta(x) dx$ and $p_d(x_t) = \int k_t(x_t | x) p_d(x) dx$\footnote{
To avoid notation overloading, we slightly abuse $p_d$ to represent the density function for both the clean target and the convolved target density $p_d*k_t$.
We distinguish them by their arguments: $p_d(x)$ denotes the original target density, while $p_d(x_t)$ refers to the convolved density. 
The also applies to the kernel $k_t$ and model density $p_\theta$.
}. 
The integration over $p_\theta(x_t)$ can be approximated using Monte Carlo integration. This involves first sampling $x' \sim p_\theta(x)$ and then $x_t \sim k_t(x_t | x')$. 
Inspired by~\citet{pooledreamfusion,wang2024prolificdreamer,luo2024diff}, we can derive the analytical gradient of~\Cref{eq:dikl:single} w.r.t $\theta$ as follows:
\begin{align}
    &\nabla_\theta \mathrm{DiKL}_{k_t}(p_\theta || p_d) = \nabla_\theta \mathrm{KL}(p_\theta * k_t || p_d * k_t) \label{eq:rkl:gradient} \\
    =&\int p_\theta(x_t) \left( \nabla_{x_t} \log p_\theta(x_t) - \nabla_{x_t} \log p_d(x_t) \right) \frac{\partial x_t}{\partial \theta} dx_t,\nonumber
\end{align}
The derivation can be found in \Cref{appendix:derivation_gradient_dikl}.
The Jacobian term $\frac{\partial x_t}{\partial \theta}$ can be efficiently computed by the vector-Jacobian product (VJP) with auto differentiation. 
However, both score functions, $\nabla_{x_t} \log p_\theta(x_t)$ and $\nabla_{x_t} \log p_d(x_t)$, in~\Cref{eq:rkl:gradient} are intractable to compute directly. 
To address this, we approximate these scores using denoising score matching (DSM)~\citep{vincent2011connection} and mixed score identity (MSI)~\citep{de2024target,phillips2024particle}, respectively.
Specifically, we estimate $\nabla_{x_t} \log p_\theta(x_t)$ by training a score network with DSM using samples from the sampler, and estimate $\nabla_{x_t} \log p_d(x_t)$ by MSI with Monte Carlo estimation.
Below, we explain these two estimators in detail.

\subsection{Estimating $\nabla_{x_t} \log p_\theta(x_t)$ with DSM}
Denoising score matching (DSM)~\citep{vincent2011connection} has been successfully used in training score-based diffusion models~\citep{song2021score}. DSM is based on the denoising score identity (DSI):
\begin{proposition}[\textbf{Denoising Score Identity}]\label{prop:dsm}
     For any convolution kernel $k(x_t|x)$, we have
\begin{align}
    \nabla_{x_t}\log p_\theta(x_t)=\int \nabla_{x_t}\log k(x_t|x) p_\theta(x|x_t) dx, \label{eq:dsi}
\end{align}
where $p_\theta(x|x_t)\propto k(x_t|x)p_\theta(x)$ is the model posterior.
\end{proposition}
See \Cref{appendix:derivation_dsi} for a proof. We can then train a time-conditioned score network $s_\phi(x_t)$ to approximate $\nabla_{x_t}\log p_\theta(x_t)$ by minimizing the score matching loss w.r.t. $\phi$:
\begin{align}
&\int \lVert s_\phi(x_t){-}\int\nabla_{x_t} \log k(x_t|x)p_\theta(x|x_t)dx\rVert_2^2 p_\theta(x_t)dx_t\nonumber\\
    =&\iint \lVert s_\phi(x_t){-}\nabla_{x_t}\log k(x_t|x)\rVert_2^2 p_\theta(x,x_t) dxdx_t {+} \text{const.},\label{eq:dsm}
\end{align}
where the equivalence can be shown by expanding the L2 norm and ignoring a term that is independent of $\phi$. The integration over $p_\theta(x,x_t)=p_\theta(x)k(x_t|x)$ can be approximated by sampling $x'\sim p_\theta(x)$ then $x_t'\sim k(x_t|x')$.
Once trained, we plug  \( s_\phi(x_t) \)  into \Cref{eq:rkl:gradient} to estimate the gradient.

\begin{figure*}[t]
    \centering
    \includegraphics[width=0.95\linewidth]{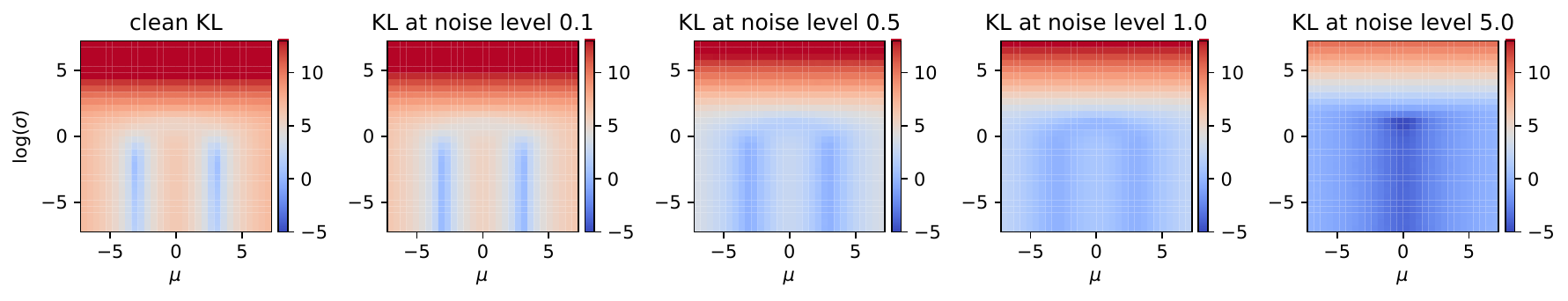}
    \caption{Heatmap of (log scale) KL divergence at different noise levels between a Gaussian model (with mean parameter $\mu$ and standard deviation parameter $\sigma$) and a two-mode MoG target in 1D.
At lower noise levels (or in the extreme case, the standard reverse KL), the divergence is highly mode-seeking, with the model favoring either one of the two modes in the target distribution.
    However, perhaps surprisingly, the KL divergence becomes more mass-covering at a higher noise level, encouraging the model to cover both modes of the target.
  }
    \label{fig:vis_why_drKL_works}
\end{figure*}

\subsection{Estimating $\nabla_{x_t} \log p_d(x_t)$ with MSI}
To estimate the gradient defined in \Cref{eq:rkl:gradient}, we also need to estimate the noisy target score $\nabla_{x_t} \log p_d(x_t)$.
Since no samples from $p_d(x)$ are available, we can no longer use DSM to estimate the score.
Fortunately, we have access to the unnormalized target density and its score function $\nabla_x\log p_d(x)=-\nabla_{x}E(x)$, which allows us to estimate this score by target score identity \citep[TSI,][]{de2024target}:
\begin{proposition}[\textbf{Target Score Identity}]\label{prop:tsm}
    For any translation-invariant convolution kernel $k(x_t|x)=k(x_t-\alpha_t x)$, we have 
\begin{align}
    \nabla_{x_t}\log p_d(x_t)=\frac{1}{\alpha_t}\int \nabla_{x} \log p_d(x) p_d(x|x_t) dx, \label{eq:tsi}
\end{align}
where $p_d(x|x_t)\propto k(x_t|x)p_d(x)$ is the target posterior.
\end{proposition}
See \Cref{appendix:derivation_tsi} for a proof. 
In practice, the TSI estimator has larger variance when the Gaussian kernel $k(x_t|x)$ has larger variance, while the DSI estimator exhibits higher variance when $k(x_t|t)$ has smaller variance. 
To address this, \citet{de2024target,phillips2024particle} propose a convex combination of the DSI and TSI to interpolate between them, favoring TSI when $k(x_t|x)$ has smaller variance and DSI when $k(x_t|x)$ has larger variance, thus minimizing the overall variance of the estimator.
We refer to this estimator as the mixed score identity (MSI).
\begin{proposition}[\textbf{Mixed Score Identity}]\label{prop:msm}
    Using a Gaussian convolution $k(x_t|x)=\mathcal{N}(x_t|\alpha_t x, \sigma_t^2I)$ with a variance-preserving (VP) scheme $\sigma_t^2=1-\alpha_t^2$, and a convex combination of TSI and DSI with coefficients $\alpha_t^2$ and $1-\alpha_t^2$, respectively, we have
\begin{multline}\label{eq:msi}
        \nabla_{x_t}\log p_d(x_t)\\=\int (\alpha_t(x+\nabla_{x} \log p_d(x))-x_t) p_d(x|x_t) dx.
\end{multline}
\end{proposition}
See \Cref{appendix:derivation_msi} for a proof. This identity estimates the score $ \nabla_{x_t}\log p(x_t)$ based on the original target score $\nabla_{x} \log p_d(x)$ which can be directly evaluated. 
We can plug it into \Cref{eq:rkl:gradient} as part of the gradient approximation.

However, to use this estimator, we also need to obtain samples from the denoising posterior $p_d(x|x_t)$ to approximate the integration over $x$ in \Cref{eq:msi}. 
We notice that the posterior $p_d(x|x_t)$ is proportional to the joint $p(x,x_t)=k(x_t|x)p(x)$, taking the form
\begin{align}
  p_d(x|x_t) \propto \exp\left(-E(x) - \lVert \alpha_t x - x_t\rVert^2/2\sigma_t^2\right),\label{eq:posterior}
\end{align}
which has a tractable score function~\citep{gao2020learning,RDMC,chen2024diffusive,grenioux2024stochastic}:
\begin{align}
    \nabla_{x}\log p_d(x|x_t) = -\nabla_{x} E(x) -\frac{\alpha_t(\alpha_t x - x_t)}{\sigma_t^2}. \label{eq:posterior:score}
\end{align}
Therefore, common score-based sampler such as HMC~\citep{duane1987hybrid,neal2011mcmc},
MALA~\citep{roberts1996exponential,roberts2002langevin}, and AIS \citep{neal2001annealed} can be directly employed to sample from the denoising posterior $p(x|x_t)$. 
Notably, compared to sampling from the original target distribution $p(x) \propto \exp(-E(x))$ using standard score-based samplers, incorporating the additional quadratic term in \Cref{eq:posterior} improves the Log-Sobolev conditions, which significantly enhances the convergence speed of samplers like ULA~\citep{vempala2019rapid,RDMC}. 

It is worth noting that it is not crucial to have a perfect posterior sampler.
This is because our method essentially works in a bootstrapping manner: the posterior samples improve the model, and a better model in turn brings the posterior samples closer to the true target. 
Having said that, accurate posterior sampling may improve the convergence rate of the model.

We summarize the whole procedure of training neural samplers with DiKL in \Cref{alg:training}\footnote{\Cref{alg:training} presents the training procedure with a batch size of 1 for clarity.}. 
In short, our training algorithm forms a nested loop: 
in the inner loop, we train a score network $s_\phi(x_t)$ to estimate the model score $\nabla_{x_t} \log p_\theta(x_t)$ with DSM;
in the outer loop, we first estimate the noisy target score $\nabla_{x_t} \log p_d(x_t)$ with MSI, and then update the neural sampler with the gradient as in \Cref{eq:rkl:gradient} using our estimated noisy target and model scores. 
One might think that this nested training procedure imposes a high computational burden.
Fortunately, we found that the inner loop typically converged within 50-100 steps in practice, minimally affecting the overall training cost.
In the following, we give an empirical illustration of how DiKL can encourage mode covering.

\begin{algorithm}[t]
    \centering
    \caption{Training Neural Samplers with DiKL}\label{alg:training}
    \begin{algorithmic}
    \Require Target $p_d(x)\propto \exp(-E(x))$, Gaussian kernels $\{(\alpha_t, \sigma_t)\}_{t=1}^T$; 
   score network training step $N_\phi$; weighting function $w(t)$; Randomly initialized $\theta$, $\phi$.
    \Repeat 
    \State {\color{gray} \# Train the score network $s_\phi(x_t)$ by DSM:}
    \For {$i \in [1, \cdots, N_\phi]$}
     \State $z\sim p(z)$, $x \leftarrow g_\theta (z)$, $t \sim \mathcal{U}\{1, \cdots, T\}$
    \State $\epsilon \sim \mathcal{N}(0, I)$, $x_t \leftarrow \alpha_t x + \sigma_t \epsilon$
    \State Update $\phi$ with $\nabla_\phi  \lVert s_\phi(x_t)-\nabla_{x_t}\log k(x_t|x)\rVert_2^2$
    \EndFor
    \State {\color{gray} \# Train the neural sampler $g_\theta(z)$ by DiKL:}
    \State $z\sim p(z)$, $x \leftarrow g_\theta (z)$
    \State $t \sim \mathcal{U}\{1, \cdots, T\}$,  $\epsilon \sim \mathcal{N}(0, I)$, $x_t \leftarrow \alpha_t x + \sigma_t \epsilon$
    \State $x'^{(1:K)} \sim p_d(x|x_t)$ 
    \Comment{posterior sampling}
    \State $d_p \leftarrow  \frac{1}{K}\sum_{k=1}^K (\alpha_t(x'^{(k)}+\nabla \log p_d(x'^{(k)}))-x_t)$\\ \Comment{MSI estimator}
    \State $\ell \leftarrow w(t) \texttt{stopgrad}(s_\phi(x_t) - d_p)^\top x_t$ \\ \Comment{surrogate loss for VJP}
    \State Update $\theta$ with  $\nabla_\theta \ell$
    \Until convergence
    \end{algorithmic}
\end{algorithm}

\subsection{DiKL Encourages mass-covering}
Unlike the mode-seeking nature of reverse KL (R-KL), DiKL promotes better mode coverage.
In this section, we provide an intuitive explanation to illustrate how this is achieved.
Assume we have a 1D Mixture of Gaussian (MoG) target with two components $p_d(x) = \frac{1}{2}\mathcal{N}(x|{-}3, 0.01) +  \frac{1}{2}\mathcal{N}(x|3, 0.01)$. 
For simplicity, we fit a 1D Gaussian model $p_\theta(x) = \mathcal{N}(x|\mu, \sigma^2)$ to this target.
As this model only contains two parameters $\theta = \{\mu, \sigma\}$, we visualize log KL and log DiKL at different noise levels against these two parameters to develop a better understanding of the reverse DiKL objective, as shown in \Cref{fig:vis_why_drKL_works}.
\par
At lower noise levels (or in the extreme case, R-KL), the divergence is highly mode-seeking, with the model favoring either one of the two modes in the target distribution. However, perhaps surprisingly, at higher noise levels, DiKL becomes more mass-covering, forcing $\mu$ to converge toward the mean of the two modes and $\sigma$ to cover both modes.
This behavior explains why DiKL encourages the model to cover more modes: higher noise levels push the model to explore adjacent modes, while lower noise levels prevent the model from forgetting previously discovered modes.
\par
Below, we demonstrate this mass-covering property on a MoG-40 target before proceeding to more complex Boltzmann distributions. 
Before presenting the results, we first outline other approaches to encouraging mode coverage, which we will use as baselines.

\begin{figure*}[!htp]
    \centering
    \begin{subfigure}[b]{0.16\linewidth}
         \includegraphics[width=\linewidth]{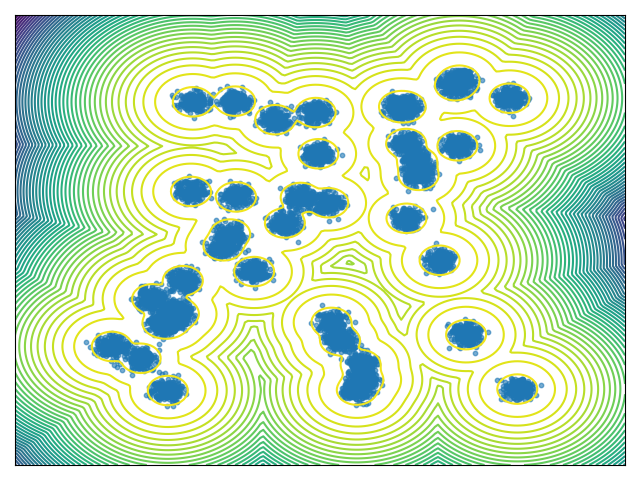}
         \caption{Ground Truth}
    \end{subfigure}
      \begin{subfigure}[b]{0.16\linewidth}
         \includegraphics[width=\linewidth]{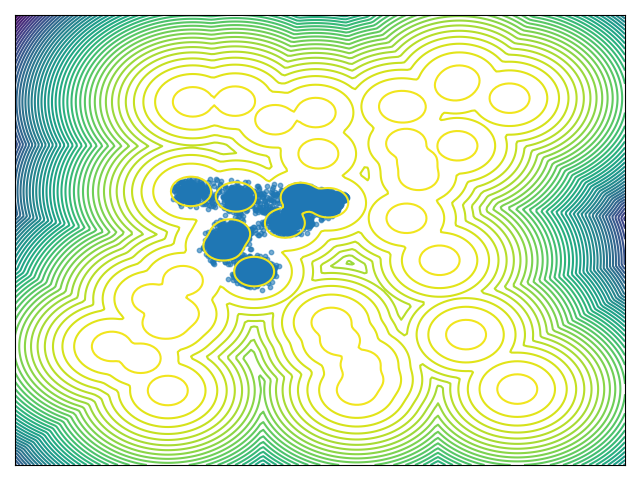}
            \caption{R-KL SM}
    \end{subfigure}
        \begin{subfigure}[b]{0.16\linewidth}
         \includegraphics[width=\linewidth]{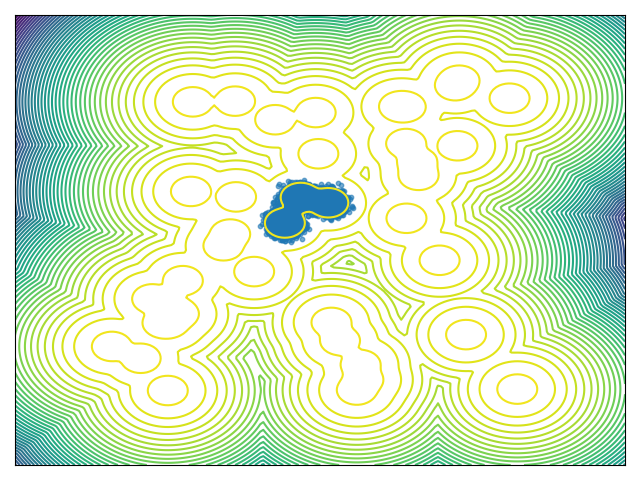}
            \caption{R-KL Bound}
    \end{subfigure}
    \begin{subfigure}[b]{0.16\linewidth}
         \includegraphics[width=\linewidth]{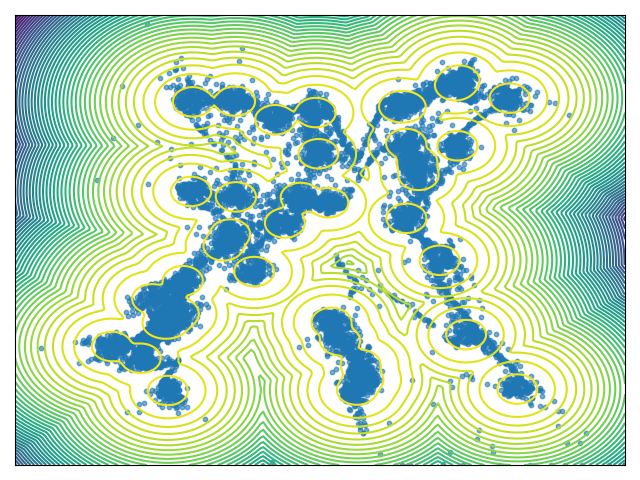}
            \caption{FAB}
    \end{subfigure}
    \begin{subfigure}[b]{0.16\linewidth}
         \includegraphics[width=\linewidth]{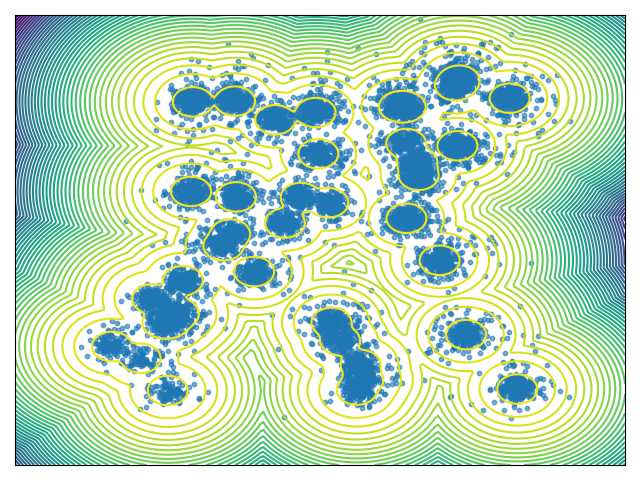}
            \caption{iDEM}
    \end{subfigure}
    \begin{subfigure}[b]{0.16\linewidth}
         \includegraphics[width=\linewidth]{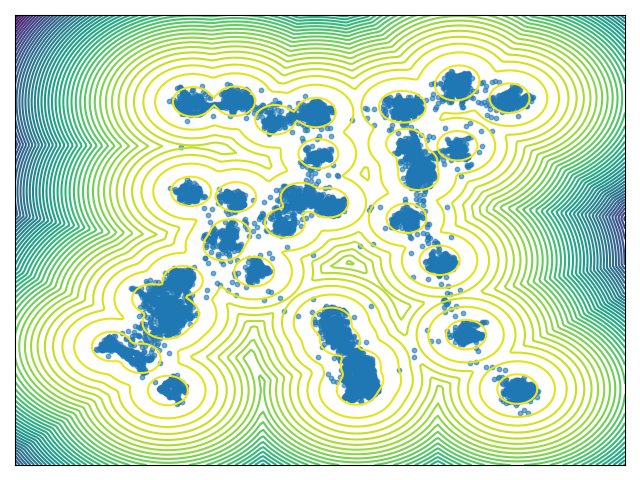}
            \caption{DiKL (ours)}
    \end{subfigure}
    \caption{Samples on MoG-40.
    We train each method for 2.5 hours, which allows all to converge.
    FAB and iDEM use replay buffers as in \citet{midgley2023flow,akhound2024iterated}.
    The high-density regions of this target are within $[-50, 50]$. 
    All methods were trained on the original scale, except for iDEM, which is normalized to \([-1, 1]\) following \citet{akhound2024iterated}.
    This normalization may simplify the task.
    }\label{fig:gmm_res}
\end{figure*}

\vspace{-0.1cm}
\section{RELATED WORKS}
\vspace{-0.1cm}
Several different types of neural samplers have been proposed in the literature, which we summarize below.

\textbf{Latent-variable-model samplers.} 
Training latent variable models as neural samplers with Fisher divergence, reverse KL (R-KL-SM) or its upper bound (R-KL Bound) has been explored in the literature~\citep{li2017gradient,shi2017kernel,zhang2019variational,song2020sliced,luo2024entropy,yin2018semi, hu2018stein}; See \Cref{eq:rkl,eq:upper_bound_rkl,eq:neg_entropy_grad}.
Such methods typically struggle for multi-modal target distributions. 
\par
\begin{table}[t]
\centering
\caption{Comparison of the log-density of samples generated by various methods, evaluated on the target density of MoG-40. ``True'' indicates the log density of true samples from the target distribution. We only report the evaluation methods that can cover all the modes, see \Cref{fig:gmm_res} for the sample visualization. 
}\label{table:gmm}
\begin{tabular}{c c c c c}
\toprule
 & \textit{True} & FAB & iDEM & DiKL (Ours) \\
\midrule
\textbf{$\log p_d(x)$}   & \textit{-6.85} & -10.74 & -8.33 & \textbf{-7.21} \\
\bottomrule
\end{tabular}
\end{table}
\textbf{Flow-based samplers.} 
Flow AIS bootstrap (FAB) \citep{midgley2023flow} is the state-of-the-art (SOTA) flow-based sampler, which is trained by minimizing the $\alpha$-2 divergence, which exhibits mass covering property. 
Note that FAB employs a prioritized replay buffer to memorize the regions that have been explored.
\par
\textbf{Diffusion/control-based samplers.} 
Several works have explored diffusion/control-based samplers. For example, the Gibbs-style sampler~\citep{grenioux2024stochastic,chen2024diffusive,zhang2023moment} constructs a forward-backward sampling procedure between the clean data space and the diffusion space.
The path integral sampler \citep[PIS,][]{zhang2022path} and the denoising diffusion sampler \citep[DDS,][]{vargas2023denoising} align the forward and backward paths by optimizing the KL divergence over the entire path measure.
GFlowNet-based sampler \citep{bengio2023gflownet,zhangdiffusion} extends these approaches to objectives with local information, like sub-trajectory balance and detailed balance.
Controlled Monte Carlo diffusions \citep[CMCD, ][]{nusken2024transport}  learns an escorted transport between interpolants from the prior distribution to the target distribution by matching the KL or log-variance divergence between the forward and backward path measures.
Non-equilibrium transport sampler \citep[NETS, ][]{albergo2024nets} learns a similar escorted transport with PINN \citep{sun2024dynamical} or Action Matching \citep{neklyudov2023action} loss.
Further improvements including combining these samplers with SMC \citep{chen2024sequential}, or incorporating MCMC to improve buffer samplers \citep{sendera2025improved}. 
However, these methods typically involve simulating the SDE by numerical integration during training, which is not scalable. 
Iterated denoising energy matching \citep[iDEM,][]{akhound2024iterated} is one of the SOTA samplers that trains a score network to approximate the noisy score of the target estimated by TSI. 
iDEM also employs a reply buffer to balance exploration and exploitation.

We compare our methods with each type of SOTA neural sampler on a mixture of 40 Gaussians (MoG-40) target in 2D following \citet{midgley2023flow}, which allows us to visually examine their mass-covering properties.
As shown in \Cref{fig:gmm_res}, our approach achieves better sample quality than all other compared neural samplers.
R-KL-based samplers struggle to capture the majority of modes due to the mode-seeking property.
FAB captures all modes but exhibits heavy density connections between modes due to the flow architecture. 
This is in contrast to our sampler which only requires using standard neural networks, which is more flexible.
As for iDEM, while it does not exhibit such connections, its samples look noisy.
This is because iDEM requires score estimation across all noise levels from the target towards a pure Gaussian distribution, which leads to high variance at larger noise levels due to TSI. 
In contrast, our approach samples directly from a generator $g_\theta$ and uses a Gaussian kernel only to connect adjacent modes, allowing for a much smaller noise level and more manageable variance. 

\textbf{Distillation for Diffusion Models.}
Variational score distillation (VSD) is a promising approach to distill knowledge from pre-trained diffusion models. 
The concept of DiKL has been successfully applied in these approaches for 3D generative models~\citep{pooledreamfusion,wang2024prolificdreamer} and diffusion distillation~\citep{luo2024diff, xie2024distillation}. 
Different from our application, they have a pre-trained diffusion model to provide the score of $p * k_t$. 
On the other hand, KL-based neural samplers \citep{li2017gradient,shi2017kernel,luo2024entropy} provide a KL estimator when there is no Gaussian convolution. 
From this perspective, we can see that our approach lies conceptually between VSD and KL-based neural samplers.

\vspace{-2pt}
\section{APPLICATION TO BOLTZMANN GENERATORS}\label{sec:boltzmann}
\vspace{-3pt}
\begin{figure*}[ht]
  \centering
  \subcaptionbox{Ground Truth}
      {\includegraphics[width=0.17\linewidth]{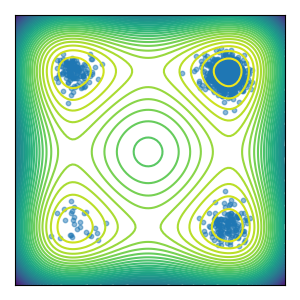}}\quad
   \subcaptionbox{KL}
      {\includegraphics[width=0.17\linewidth]{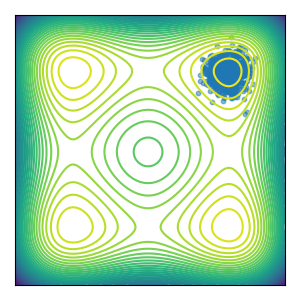}}\quad
   \subcaptionbox{FAB}
      {\includegraphics[width=0.17\linewidth]{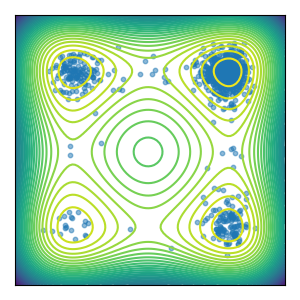}}\quad
       \subcaptionbox{iDEM}
      {\includegraphics[width=0.17\linewidth]{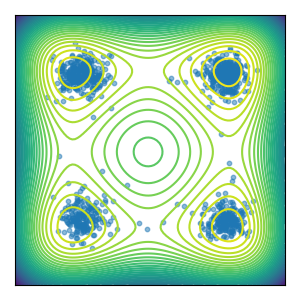}}\quad
       \subcaptionbox{DiKL (ours)}
      {\includegraphics[width=0.17\linewidth]{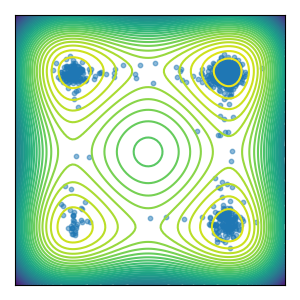}}
    \caption{2D marginal (1st and 3rd dimensions) of samples from MW-32.
    Our approach and FAB manage to find all the modes with correct weights, iDEM finds all modes but with wrong weights, and the neural sampler trained with standard KL divergence only capture one mode.\vspace{10pt}
    }
    \label{fig:manywell_vis}
    \vspace{-0.3cm}

\end{figure*}
\begin{figure}[t]
    \centering
    \includegraphics[width=\linewidth]{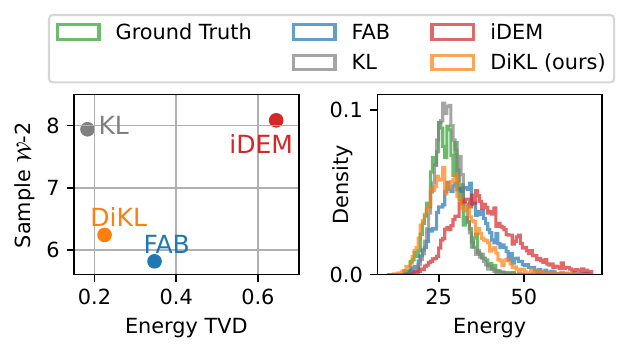}
    \caption{\textbf{Left. }Wasserstein-2 ($\mathcal{W}$-2) distance of samples and total variation distance (TVD) of energy on MW-32.   Our method and FAB clearly outperform iDEM and KL in this evaluation. \textbf{Right. }Histogram of sample energy. Our approach outperforms both FAB and iDEM. 
     Note that although the KL approach yields better energy, it captures only one mode, as shown in \Cref{fig:manywell_vis}.
    }
    \label{fig:MW_vis}
\end{figure}
One important application of neural samplers is to generate samples from Boltzmann distributions, where the target distribution defines the probability density that a system will be in a certain state as a function of that state's energy and the temperature of the system.
This type of neural sampler is also known as Boltzmann Generator \citep{noe2019boltzmann}.
In the following, we will omit the temperature for simplicity, as it can be absorbed into the energy function.
\par
In this section, we use our neural sampler $g_{\theta}$ as a Boltzmann Generator to generate samples from $n$-body systems\footnote{Code for all experiments is available at \url{https://github.com/jiajunhe98/DiKL}.},
where the energy is defined over the pairwise distances between $n$ particles.
These systems can be defined in either internal or Cartesian coordinates.
Note that, for Cartesian coordinates, the energy of the system will remain invariant if we apply rotation, reflection, translation, and permutation to the entire system.
Formally, representing each configuration of the system by a matrix $X \in \mathbb{R}^{n\times d}$, our target distribution $p_d(X)$ is invariant to the product group of the Euclidean group and the Symmetric group of degree $n$, i.e. $G = \text{E}(d) \times \mathbb{S}_n$.
\par
This invariance presents a challenge when training the neural sampler. 
Recall the sampler learns a mapping \( g_\theta: \mathcal{Z} \rightarrow \mathcal{X} \), where \( \mathcal{X} \) represents the space of system configurations, and \( \mathcal{Z} \) represents the latent space. 
If \( \mathcal{X} \) includes configurations with symmetries but \( \mathcal{Z} \) does not account for these symmetries, the network would need to model every equivariant configuration separately (for example, the model would need to assign same density for the configurations in one equivariant class respect to $G$), leading to inefficient training.
\par
On the other hand, we can parameterize the neural sampler $g_{\theta}$ with an Equivariant Graph Neural Networks \citep[EGNN, ][]{satorras2021n,hoogeboom2022equivariant}, ensuring that $g_\theta$ is $G$-equivariant. 
\begin{proposition}\label{prop:g_eqvar}
    Let the neural sampler \( g_\theta: \mathcal{Z} \rightarrow \mathcal{X} \) be an $G$-equivariant mapping. 
    If the distribution $p(Z)$ over the latent space $\mathcal{Z}$ is $G$-invariant, then $p_{\theta}(X) = \int \delta(X -  g_\theta(Z)) p(Z) dZ$ is $G$-invariant.
\end{proposition}
The proof can be found in \Cref{appendix:neural_sampler_theorem}. Therefore, our neural sampler does not need to explicitly model the invariance, simplifying the training process.
\par
However, another challenge arises from translation.
As noted by \citet{midgley2024se}, there does not exist a translation invariant probability measure in Euclidean space.
Therefore, following \citet{midgley2024se,satorras2021n,hoogeboom2022equivariant,akhound2024iterated}, we constrain both $\mathcal{X}$ and  $\mathcal{Z}$ to be the subspace of $\mathbb{R}^{n\times d}$ with zero center of mass, i.e., $X^\top 1 =Z^\top 1= 0$.
This allows us to embed the product group in $n\times d$ into an orthogonal group in $nd$-dimensional space: $\text{E}(d) \times \mathbb{S}_n 	\hookrightarrow \text{O}(nd)$.
\par
Having decided on the architecture for the neural sampler $g_{\theta}$ and tackled the translation invariance, we now consider the scoring network $s_{\phi}$ for the neural sampler.
According to \citet[][Lemma 2]{papamakarios2021normalizing}, \emph{if $G$ is a subgroup of the orthogonal group, then the gradient of a $G$-invariant function is $G$-equivariant.}
 Therefore, the score network for the neural sampler which defines the 
$G$-invariant distribution is $G$-equivariant. 
 To achieve this, we train an EGNN score network with denoising score matching (DSM) within the zero-centered subspace, following \citet{hoogeboom2022equivariant}.
\par
Additionally, when both the model and target density are $G$-invariant, the identity  $\nabla \log p_\theta(X_t) - \nabla \log p_d(X_t)$ in the gradient of reverse DiKL as shown in \Cref{eq:rkl:gradient} should also be $G$-equivariant.
This necessitates a $G$-equivariant MSI estimator for $\int (\alpha_t(X{+}\nabla \log p_d(X)){-}X_t) p_d(X|X_t) dx$.
Fortunately, this holds true for a broad class of estimators under mild conditions, including importance sampling and AIS estimators, with different choices of samplers for $p_d(X|X_t)$, such as MALA and HMC.
Detailed discussion can be found in \Cref{appendix:invariance_theorem}.

\subsection{Experiments and Results}

\par
We evaluate our approach on three distinct Boltzmann distributions: Many-Well-32 in the internal coordinate, and Double-Well-4 and Lennard-Jones-13 in the Cartesian coordinate. 
These tasks aim to provide a comprehensive evaluation of highly multi-modal targets, encompassing both balanced and imbalanced modes, as well as energy functions exhibiting invariance.
Detailed setups can be found in \Cref{appendix:experiment_setup}.
\par
\textbf{Internal MW-32.}
\citet{midgley2023flow} introduced this target by stacking 2D Double-Well 32 times, forming a distribution with $2^{32}$ modes in total.
These modes carry different weights. 
Therefore, this task can assess whether each method successfully covers all modes and accurately captures their weights.
\par
We report the Wasserstein-2 ($\mathcal{W}$-2) distances between samples yielded by these methods and ground truth samples obtained by MCMC.
We also evaluate the sample energy and report the total-variant distance (TVD) between the distribution of the energy of samples.
We note that both metrics have their limitations: 
$\mathcal{W}$-2 tends to be less sensitive to noisy samples, which can be particularly detrimental in some 
$n$-body systems. 
Conversely, the energy TVD is less sensitive to missing modes. 
To provide a comprehensive evaluation, we plot both metrics together in \Cref{fig:MW_vis}, which shows that our method and FAB clearly outperform iDEM and KL in this evaluation. We also visualize the samples along two selected axes to assess mode coverage in \Cref{fig:manywell_vis}, showing that only our approach and FAB can find all modes with roughly correct weights.

\begin{figure}
    \centering
    \begin{subfigure}{\linewidth}
    \centering
        \includegraphics[width=\linewidth]{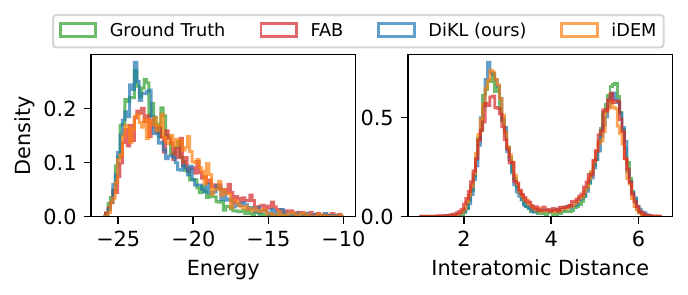}
        \caption{DW-4}
    \end{subfigure}
    \begin{subfigure}{\linewidth}
    \centering
        \includegraphics[width=\linewidth]{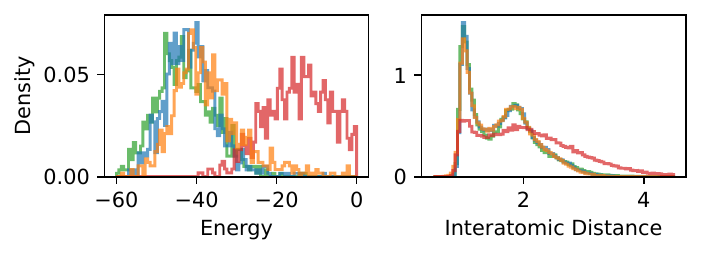}
        \caption{LJ-13}
    \end{subfigure}
    \caption{Histogram of sample energy and interatomic distance on Cartesian DW-4 and LJ-13. 
    Our approach achieves comparable performance on both tasks to iDEM, with only 1 number of function evaluation (NFE), while iDEM requires 1,000 NFEs.
    }
    \label{fig:ljdw}
    \vspace{-0.5cm}
\end{figure}

\begin{table*}
\centering
\caption{Comparison of our approach with FAB and iDEM on Cartisian DW-4 and LJ-13.
We report the Wasserstein-2 ($\mathcal{W}$-2) distance of samples, and total variation distances (TVDs) of energy and atomic distances.  
We evaluate each metric using 5,000 samples repeated ten times and report the mean and standard deviation.
}\label{tab:res_SE3}
\resizebox{\textwidth}{!}{%
\begin{tabular}{@{}ccccccc@{}}
\toprule
 & \multicolumn{3}{c}{\textbf{Cartesian DW-4 (8D)}} & \multicolumn{3}{c}{\textbf{Cartesian LJ-13 (39D)}} \\ \cmidrule(r){2-4}\cmidrule(r){5-7} 
 & {\small\shortstack{Sample $\mathcal{W}$-2}} & {\small\shortstack{Energy TVD}} & {\small\shortstack{Distance TVD}} & {\small\shortstack{Sample $\mathcal{W}$-2}} & {\small\shortstack{Energy TVD}} & {\small\shortstack{Distance TVD}} \\ \midrule
 \textbf{FAB} & \textbf{1.554 $\pm$ 0.015} & 0.224 $\pm$ 0.008 & \textbf{0.097 $\pm$ 0.005} &  4.938 $\pm$ 0.009 &  0.902 $\pm$ 0.010  &  0.252 $\pm$ 0.002 \\
\textbf{iDEM} & 1.593 $\pm$ 0.012 & \underline{0.197 $\pm$ 0.010} & 0.103 $\pm$ 0.005 & \textbf{4.172 $\pm$ 0.007} & \underline{0.306 $\pm$ 0.013}  & \underline{0.044 $\pm$ 0.001} \\
\textbf{DiKL (ours)} & \underline{1.581 $\pm$ 0.026} & \textbf{0.167 $\pm$ 0.012} & \underline{0.101 $\pm$ 0.006} & \underline{4.233 $\pm$ 0.008} & \textbf{0.239 $\pm$ 0.019} & \textbf{0.042 $\pm$ 0.002} \\ \bottomrule
\end{tabular}%
}
\end{table*}

\textbf{Cartesian DW-4 and LJ-13.} 
\citet{kohler2020equivariant} introduced these two tasks to access the model under invariance target distributions. 
Specifically, the target energy is invariant to the product group $G=\text{E}(d)\times \mathbb{S}_n$ where $d=2$ and 3 for DW and LJ, respectively.
\par
We compare our approach with FAB and iDEM.
We report the $\mathcal{W}$-2 distance for samples, TVD for energy, and TVD for interatomic distance in \Cref{tab:res_SE3}.
We also visualize histograms for sample energy and interatomic distance in \Cref{fig:ljdw}.
As shown, our method achieves competitive performance with both FAB and iDEM on DW-4, and notably outperforms FAB on LJ-13.
\par
\textbf{Training and sampling time.} 
We report the training and sampling times for our approach and baselines in \Cref{tab:time}. 
Notably, our method shows faster training and sampling times compared to both FAB and iDEM. 
FAB depends on a large and limited normalizing flow, which leads to significantly longer training times, particularly for complex tasks like LJ. 
In contrast, our approach maintains consistent training times across different tasks. 
iDEM is diffusion-based and requires intensive computation for sampling, which is 1,000 times slower than our approach.
\par
\textbf{Discussion.}
Our method is comparable to SOTA Boltzmann generators, including FAB and iDEM, and has both faster training and sampling times. 
Additionally, we highlight that, unlike iDEM and FAB, which use replay buffers for exploration-exploitation balance, our approach achieves this \emph{without} relying on any replay buffer, offering a more clean, straightforward and easy-to-extend solution.

\begin{table}
\centering
\caption{Training and sampling wall-clock times for FAB, iDEM and our sampler.
We measure this on a single NVIDIA A100 (80GB) GPU. 
We omit the sampling times for FAB on DW-4 and LJ13 as it is implemented in JAX with JIT compilation, making direct comparison with the other methods implemented in PyTorch not feasible. 
However, we expect FAB to have slightly slower sampling times than our method due to its large flow network.}\label{tab:time}
\resizebox{\linewidth}{!}{%
\begin{tabular}{@{}ccccc@{}}
\toprule
  &  & \textbf{FAB} & \textbf{iDEM} & \textbf{DiKL (ours)} \\ \midrule
\multirow{3}{*}{\textbf{Training}} & MW-32 & 3.5h  & 3.5h  & \textbf{2.5h}  \\
 & DW-4 & 4.5h  & 4.5h  & \textbf{0.9h} \\
 & LJ-13 & 21.5h & \textbf{6.5h}  &  \textbf{6.5h} \\ \midrule
\multirow{3}{*}{\textbf{\shortstack{Batch Sampling \\ (1,000 samples)}}} & MW-32 &  \textbf{0.01s} & 7.2s  & \textbf{0.01s}  \\
 & DW-4 & - & 2.6s & \textbf{0.01s}  \\
 & LJ-13 & - &19.7s  & \textbf{0.02s}  \\ \bottomrule
\end{tabular}%
}
\end{table}

\section{CONCLUSION}
In this work, we proposed a new training paradigm for neural samplers using reverse diffusive KL divergence, providing a simple yet efficient method to achieve the mass-covering property. We demonstrated its effectiveness on both synthetic and $n$-body system targets. 
Our approach matches or outperforms SOTA methods like FAB and iDEM, with improved training and sampling efficiency and without relying on replay buffer.

While our method achieves comparable performance with SOTA flow-based and diffusion-based samplers with faster training and sampling speed, it has a few limitations as discussed below.

\textbf{Model density.} The density $p_{\theta}(x)$ of our neural sampler is intractable since we have a latent variable model with a nonlinear and non-invertible generator $g_{\theta}(z)$. Compared to FAB, although we have a more flexible generator, we cannot use importance re-weighting to correct the potential bias of generated samples since we do not have access to the density of our samples.

\textbf{Model flexibility.} Although our one-step generator $g_{\theta}(z)$ has significantly faster sampling speed than diffusion-based samplers, it has limited model flexibility compared to multi-step diffusion models such as iDEM. 
It is, therefore, more difficult for our model to handle more complicated energy functions.
For example, DiKL cannot capture more complicated target LJ-55 well, as shown in \Cref{appendix:lj55}. 
Therefore, a promising direction is to design approaches to balance the sampling speed and model expressiveness. 

\textbf{Training stability.} For the two $n$-body system targets in the Cartesian coordinate (i.e., DW-4 and LJ-13), training can be unstable near convergence. We employed some criteria on the energies of model samples to perform early stopping. For future work, it might be beneficial to employ a replay buffer similar to the one used in FAB or iDEM to balance exploration and exploitation and stabilize training.

\textbf{Posterior sampling.} Posterior sampling could be a bottleneck in our training procedure, which requires running MCMC to sample from the denoising posterior distribution $p_d(x|x_t)$ at each training iteration. The MCMC sampling procedure might have a slow mixing speed for complicated target distributions. In fact, we note that most diffusion-based samplers have this bottleneck since there is no data available to train a neural network denoiser, unlike diffusion or score-based generative models. 

\clearpage
\section*{Acknowledgments}
JH is supported by the University of Cambridge Harding Distinguished Postgraduate Scholars Programme.
JH and JMHL acknowledge support from a Turing AI Fellowship under grant EP/V023756/1.
MZ and DB acknowledge funding from AI Hub in Generative Models, under grant EP/Y028805/1.
\par
Part of this work was performed using resources provided by the Cambridge Service for Data Driven Discovery (CSD3) operated by the University of Cambridge Research Computing Service (\url{www.csd3.cam.ac.uk}), provided by Dell EMC and Intel using Tier-2 funding from the Engineering and Physical Sciences Research Council (capital grant EP/T022159/1), and DiRAC funding from the Science and Technology Facilities Council (\url{www.dirac.ac.uk}).

\bibliography{ref}

\onecolumn
\aistatstitle{
Training Neural Samplers with Reverse Diffusive KL Divergence: \\
Supplementary Materials}
\thispagestyle{plain}
\tableofcontents

\appendix

\section{DERIVATION OF ANALYTICAL GRADIENT FOR REVERSE DIKL}\label{appendix:derivation_gradient_dikl}
The gradient of reverse DiKL w.r.t. the model parameter $\theta$ is given by
\begin{equation}
    \nabla_\theta \mathrm{DiKL}_{k_t}(p_\theta || p_d) 
    =\int p_\theta(x_t) \left( \nabla_{x_t} \log p_\theta(x_t) - \nabla_{x_t} \log p_d(x_t) \right) \frac{\partial x_t}{\partial \theta} dx_t.
\end{equation}
\begin{proof}
    The reverse DiKL at time $t$ is defined as
    \begin{equation}
        \mathrm{DiKL}_{k_t}(p_\theta || p_d) 
        = \int  \left( \log p_\theta(x_t) - \log p_d(x_t) \right) p_\theta(x_t) dx_t.
    \end{equation}
    We first reparameterize $x_t$ as a function of $z$ and $\epsilon$:
    \begin{equation}
        x_t = \alpha_t g_{\theta}(z) + \sigma_t \epsilon_t\equiv h_{\theta}(z,\epsilon_t),
    \end{equation}
    where $z\sim p(z)\equiv \mathcal{N}(z|0,I)$ and $\epsilon_t\sim p(\epsilon_t)\equiv\mathcal{N}(\epsilon_t|0,I)$. It then follows that
    \begin{align}
        \nabla_\theta \mathrm{DiKL}_{k_t}(p_\theta || p_d)
        &= \nabla_\theta \int \left( \log p_\theta(x_t) - \log p_d(x_t) \right) p_\theta(x_t) dx_t \\
        &= \nabla_\theta \iiint \left( \log p_\theta(x_t) - \log p_d(x_t) \right)\delta(x_t-h_{\theta}(z,\epsilon_t)) p(z) p(\epsilon_t) dx_tdzd\epsilon\\
        &= \nabla_\theta \iint \left( \log p_\theta(x_t) - \log p_d(x_t) \right)|_{x_t=h_{\theta}(z,\epsilon_t)} p(z) p(\epsilon_t) dzd\epsilon\\
        &=\iint \left.\left(\nabla_\theta \log p_\theta(x_t)+\nabla_{x_t} \log p_\theta(x_t)\frac{\partial x_t}{\partial \theta} - \nabla_{x_t} \log p_d(x_t)\frac{\partial x_t}{\partial \theta} \right)\right|_{x_t=h_{\theta}(z,\epsilon_t)} p(z)p(\epsilon_t) dzd\epsilon \\
        &= \int \left(\nabla_\theta \log p_\theta(x_t)+\nabla_{x_t} \log p_\theta(x_t)\frac{\partial x_t}{\partial \theta} - \nabla_{x_t} \log p_d(x_t)\frac{\partial x_t}{\partial \theta} \right) p_{\theta}(x_t)dx_t \\
        &= \int \left(\nabla_{x_t} \log p_\theta(x_t)\frac{\partial x_t}{\partial \theta} - \nabla_{x_t} \log p_d(x_t)\frac{\partial x_t}{\partial \theta} \right) p_{\theta}(x_t)dx_t,
    \end{align}
    where the last line follows since
    \begin{equation}
        \int \nabla_{\theta}\log p_\theta(x_t)p_\theta(x_t) dx_t = \int \nabla_{\theta} p_\theta(x_t)dx_t =\nabla_{\theta}\int  p_\theta(x_t)dx_t = \nabla_{\theta} 1=0.
    \end{equation}
    This completes the proof.
\end{proof}

\section{DERIVATIONS OF SCORE IDENTITIES}

\subsection{Derivation of Denoising Score Identity (DSI)}\label{appendix:derivation_dsi}
\begin{manualprop}{\ref{prop:dsm}}[\textbf{Denoising Score Identity}]
     For any convolution kernel $k(x_t|x)$, we have
\begin{align}
    \nabla_{x_t}\log p_\theta(x_t)=\int \nabla_{x_t}\log k(x_t|x) p_\theta(x|x_t) dx, %
\end{align}
where $p_\theta(x|x_t)\propto k(x_t|x)p_\theta(x)$ is the model posterior.
\end{manualprop}

\begin{proof}
It follows that
\begin{align}
    \nabla_{x_t}\log p_\theta(x_t)
    & = \frac{\nabla_{x_t} p_\theta(x_t)}{p_\theta(x_t)} \\
    &=\frac{\nabla_{x_t}\int k(x_t|x)p_\theta(x)dx}{p_\theta(x_t)}\\
    &=\frac{\int \nabla_{x_t}k(x_t|x)p_\theta(x)dx}{p_\theta(x_t)}\\
    & =\frac{\int \nabla_{x_t}\log k(x_t|x) k(x_t|x)p_\theta(x)dx}{p_\theta(x_t)}\\
    & =\int \nabla_{x_t}\log k(x_t|x) p_\theta(x|x_t) dx.
\end{align}
Note that the same argument can be used to derive DSI for the target distribution $p_d(x)$:
\begin{equation}
    \nabla_{x_t}\log p_d(x_t) = \int \nabla_{x_t}\log k(x_t|x) p_d(x|x_t) dx.
\end{equation}
\end{proof}

\subsection{Derivation of Target Score Identity (TSI)}\label{appendix:derivation_tsi}
\begin{manualprop}{\ref{prop:tsm}}[\textbf{Target Score Identity}]
    For any translation-invariant convolution kernel $k(x_t|x)=k(x_t-\alpha_t x)$, we have 
\begin{align}
    \nabla_{x_t}\log p_d(x_t)=\frac{1}{\alpha_t}\int \nabla_{x} \log p_d(x) p_d(x|x_t) dx, %
\end{align}
where $p_d(x|x_t)\propto k(x_t|x)p_d(x)$ is the target posterior.
\end{manualprop}
\begin{proof}
Since $k(x_t|x)=k(x_t-\alpha_t x)$ is translation-invariant, we have
\begin{equation}
    \nabla_{x_t}\log k(x_t|x) = -\alpha^{-1}\nabla_{x}\log k(x_t|x).
\end{equation}
But by Bayes rule, we have 
\begin{equation}
\nabla_{x}\log k(x_t|x)=\nabla_{x}\log p_d(x|x_t)-\nabla_{x} \log p_d(x).
\end{equation}
Using DSI, it then follows that 
\begin{align}
    \nabla_{x_t}\log p_d(x_t)
    &=\int \nabla_{x_t} k(x_t|x)p_d(x|x_t)dx\\
    &=-\alpha^{-1}\int \nabla_{x} k(x_t|x)p_d(x|x_t)dx\\
    &=\alpha^{-1}\int\left(\nabla_x\log p_d(x) - \nabla_x \log p_d(x|x_t)\right)p_d(x|x_t)dx\\
    &=\alpha^{-1}\int\nabla_x\log p_d(x)p_d(x|x_t)dx,
\end{align}
where the last equality follows since
\begin{align}
\int \nabla_{x}\log p_d(x|x_t)p_d(x|x_t) dx = \int \nabla_{x} p_d(x|x_t)dx =\nabla_{x}\int  p_d(x|x_t)dx = \nabla_{x} 1=0.
\end{align}
This completes the proof.
\end{proof}

\subsection{Derivation of Mixed Score Identity (MSI)}\label{appendix:derivation_msi}
\begin{manualprop}{\ref{prop:msm}}[\textbf{Mixed Score Identity}]
    Using a Gaussian convolution $k(x_t|x)=\mathcal{N}(x_t|\alpha_t x, \sigma_t^2I)$ with a variance-preserving (VP) scheme $\sigma_t^2=1-\alpha_t^2$, and a convex combination of TSI and DSI with coefficients $\alpha_t^2$ and $1-\alpha_t^2$, respectively, we have
\begin{equation}%
        \nabla_{x_t}\log p_d(x_t)=\int (\alpha_t(x+\nabla_{x} \log p_d(x))-x_t) p_d(x|x_t) dx.
\end{equation}
\end{manualprop}
\begin{proof}
    For a Gaussian convolution kernel $k(x_t|x)=\mathcal{N}(x_t|\alpha_t x, \sigma_t^2I)$, DSI becomes
    \begin{align}
        \nabla_{x_t}\log p_d(x_t)
        & =\int \nabla_{x_t}\log k(x_t|x) p_d(x|x_t) dx \\
        & = \int \nabla_{x_t} \left( -\frac{\lVert x_t - \alpha x \rVert^2}{2\sigma_t^2} \right)p_d(x|x_t)dx \\
        &= \int \left(\frac{\alpha x - x_t}{\sigma_t^2}\right) p_d(x|x_t)dx.
    \end{align}
    Since $\sigma_t^2=1-\alpha_t^2$, it then follows that
    \begin{align}
        \nabla_{x_t}\log p_d(x_t)
        & = \int \left( \alpha_t^2 \frac{\nabla_x \log p_d(x)}{\alpha} + (1-\alpha_t^2) \frac{\alpha_t x - x_t}{\sigma_t^2} \right)p_d(x|x_t)dx \\
        & = \int (\alpha_t(x+\nabla_{x} \log p_d(x))-x_t) p_d(x|x_t) dx.
    \end{align}
    This completes the proof.
\end{proof}

\section{PROOFS REGARDING INVARIANCES AND EQUIVARIANCES}\label{appendix:invariance_theorem}
\subsection{Proof of \Cref{prop:g_eqvar}}\label{appendix:neural_sampler_theorem}
We first prove \Cref{prop:g_eqvar} in the main text.
This proposition tells us for is a $G$-equivariant neural sampler which maps from latent space to sample space, if the latent space is invariant then the sample space is also invariant.
We restate the formal statement below.
\begin{manualprop}{\ref{prop:g_eqvar}}
 Let the neural sampler \( g_\theta: \mathcal{Z} \rightarrow \mathcal{X} \) be an $G$-equivariant mapping. 
    If the distribution $p(Z)$ over the latent space $\mathcal{Z}$ is $G$-invariant, then $p_{\theta}(X) = \int \delta(X -  g_\theta(Z)) p(Z) dZ$ is $G$-invariant.
\end{manualprop}
\begin{proof}
For a transformation $G$\footnote{We slightly abuse the notation by using $G$ to represent both the set $G = \text{E}(3) \times \mathbb{S}_n$ and its elements.}, we have
    \begin{align}
        p_{\theta}(G \circ X) 
        &= \int \delta(G \circ X-  g_\theta(Z)) p(Z)dZ\\
        &= \int  \delta(X -  g_\theta(G^{-1}\circ  Z)) p(G^{-1}\circ  Z) dZ\\
 &=  \int  \delta(X - g_\theta(Z')) p(Z') dZ \\
 &=\int  \delta(X -  g_\theta(Z')) p(Z') dZ' \\
 &= p_{\theta}(X),
    \end{align}
    where $Z'=G^{-1} \circ Z$, and the penultimate line follows since the transformation in $\text{E}(d) \times \mathbb{S}_n$ preserves the volume.
\end{proof}
\subsection{Monte Carlo Score Estimators are $G$-equivariant}
Recall that, in \Cref{sec:boltzmann}, we discussed that we need a $G$-equivariant MSI estimator for $\int (\alpha_t(X{+}\nabla \log p_d(X)){-}X_t) p_d(X|X_t) dx$, and we mentioned that this holds true for a broad class of estimators under mild conditions, including importance sampling and AIS estimators, with different choices of samplers for $p_d(X|X_t)$, such as MALA and HMC.
We now provide a detailed discussion below.
\par
Before tackling the equivariance, first recall that we embed the product group in $n\times d$ into an orthogonal group in $nd$-dimensional space: $\text{E}(d) \times \mathbb{S}_n \hookrightarrow \text{O}(nd)$ by constrain both $\mathcal{X}$ to be the subspace of $\mathbb{R}^{n\times d}$ with zero center of mass. 
Therefore, both $X$ and $X_t$ are zero-centered, i.e., $X^\top 1 = X_t^\top 1= 0$.
Additionally, the score $\nabla \log p_d(X)$ is also zero-centered.
The following Proposition provides a formal statement.
\begin{proposition}\label{prop:translation}
    Let $X \in \mathbb{R}^{n\times d}$ be a random variable representing a $d$-dimensional $n$-body system.
    The gradient of a translation invariant energy function $f$, i.e, $f(X) = f(X+ 1t^\top)$ is (1) translation invariant; and (2) 0-centered, i.e., $\nabla f(X)^\top  1 = 0$.
\end{proposition}
\begin{proof}
(1) We first prove that the gradient is translation invariant.
Given the fact that the function is translation invariant:
    \begin{align}
        f(X) = f(X+1t^\top),
    \end{align}
taking gradient w.r.t. $X$ on both sides, we have
    \begin{align}
       \nabla_X  f(X) =  \nabla_X f(X+1t^\top).
    \end{align}
By chain rule, we also have
    \begin{align}
       \nabla_{X}
f(X+1t^\top) &=  \nabla_{X+1t^\top}
f(X+1t^\top)\nabla_X (X+1t^\top)= \nabla_{X+1t^\top}
f(X+1t^\top).
    \end{align}
But let  $X'=X+1t^\top$, we have
    \begin{align}
         \nabla_X f(X) = \nabla_{X'} 
f(X'),
    \end{align}
    which shows that the gradient of a translation invariant function is also translation invariant.
    \par
   (2) Now we show that the gradient is always $0$-centered.
   Applying first-order Taylor expansion at $X_0$ to both sides of $f(X) = f(X+1t^\top)$, we have
    \begin{align}
        &f(X_0)  + \text{tr}\left(\nabla f(X_0)(X - X_0)^\top\right) \\= &\underbrace{f(X_0+1t^\top)}_{=f(X_0)} +\text{tr}\left(\underbrace{\nabla f(X_0+1t^\top)}_{=\nabla f(X_0)}(X - X_0 - 1t^\top)^\top\right) .
    \end{align}
    Therefore, we have
    \begin{align}
        & \text{tr}\left(\nabla f(X_0)(X - X_0)^\top\right) = \text{tr}\left(\nabla f(X_0+1t^\top)(X - X_0 - 1t^\top)^\top\right). 
    \end{align}
    It then follows that
      \begin{align}
        &\text{tr}\left(\nabla f(X_0) t1^\top\right) = 0.
    \end{align}
   By the property of the trace operator, we have
    \begin{align}
         \text{tr}\left(t^\top \nabla f(X_0)^\top  1 \right) = t^\top \nabla f(X_0)^\top 1 = 0.
    \end{align}
  Note that the equation holds for any $t$. Therefore, we have
    \begin{align}
       \nabla f(X_0)^\top 1= 0.
    \end{align}
    This completes the proof.
\end{proof}
Therefore, the entire MSI: $\nabla_{X_t}\log p_d(X_t)=\int (\alpha_t(X{+}\nabla_X \log p_d(X)){-}X_t) p_d(X|X_t) dx$ is \textbf{(1) zero-centered} and \textbf{(2) $G$-equivariant w.r.t $X_t$}.
We now need to prove that the aforementioned MC estimators and samplers (IS/AIS with HMC/MALA) meet these requirements.

\subsubsection{Importance Sampling}
We begin with the simplest case, where we estimate $\int (\alpha_t(X{+}\nabla \log p_d(X)){-}X_t) p_d(X|X_t) dx$ with Importance Sampling (IS) with zero-centered Gaussian proposal, as employed in \citet{akhound2024iterated}.
Formally, we use the IS estimator (or more precisely, self-normalized IS estimator) as follows:
\begin{align}
   & \int (\alpha_t(X{+}\nabla \log p_d(X)){-}X_t) p_d(X|X_t) dx\\
    = &\alpha_t \frac{\int  (X{+}\nabla \log p_d(X) )p_d(X) \bar{\mathcal{N}}(X_t|\alpha_t X, \sigma_t^2)}{\int p_d(X) \bar{\mathcal{N}}(X_t|\alpha_t X, \sigma_t^2) dx} dx - X_t\\
   =  &\alpha_t \frac{\int  (X{+}\nabla \log p_d(X) )p_d(X) \bar{\mathcal{N}}(X| X_t/\alpha_t, \sigma_t^2/\alpha_t^2)}{\int p_d(X) \bar{\mathcal{N}}(X| X_t/\alpha_t, \sigma_t^2/\alpha_t^2) dx} dx - X_t\\
   \approx& \alpha_t \frac{\sum_{n} (X^{(n)}+\nabla \log p_d(X^{(n)}) ) p_d(X^{(n)}) }{\sum_n p_d(X^{(n)}) }- X_t, \quad X^{(1:N)} \sim \bar{\mathcal{N}}(X| X_t/\alpha_t, \sigma_t^2/\alpha_t^2).
\end{align}
Here, we slightly abuse the notation of the Gaussian distribution:  the random variable and its mean are both in matrix form, while the variance is a scalar. 
We use this notation to represent an isotropic Gaussian for matrices, i.e.,
$\mathcal{N}(X|Y, v) = \mathcal{N}(\text{vec}(X)|\text{vec}(Y), vI)$.
Unless otherwise specified, we will adhere to this notation throughout our proof.
We also use $\bar{\mathcal{N}}$ to denote Gaussian in the zero-centered subspace, i.e., \begin{align}
    \bar{\mathcal{N}}(X|\cdot) \propto \begin{cases}
        \mathcal{N}(X|\cdot), & \text{if } X^\top 1 = 0,\\
        0, & \text{otherwise.}
    \end{cases}
\end{align}
In other words, we draw samples from the proposal $\bar{\mathcal{N}}(X| X_t/\alpha_t, \sigma_t^2/\alpha_t^2)$, and target at $p_d(X) \bar{\mathcal{N}}(X_t|\alpha_t X, \sigma_t^2)$.
The importance weight is given by
\begin{align}
    w(X) = \frac{p_d(X) \bar{\mathcal{N}}(X_t|\alpha_t X, \sigma_t^2)}{\bar{\mathcal{N}}(X| X_t/\alpha_t, \sigma_t^2/\alpha_t^2) } \propto p_d(X).
\end{align}
\par
It is easy to check that this estimator is zero-centered.
We now prove that this estimator is $G$-equivariant.
Our proof follows \citet{akhound2024iterated} closely.
We mostly restate their proof here just for completeness.
\begin{proof}
Assume we apply some transformation $G$ to $X_t$, the estimator becomes
   \begin{align}
       &\alpha_t \frac{\sum_{n} (G \circ X^{(n)}+\nabla \log p_d(G \circ X^{(n)}) ) p_d(G \circ X^{(n)}) }{\sum_n p_d(G \circ X^{(n)}) }- G\circ X_t, \quad G \circ  X^{(1:N)} \sim \bar{\mathcal{N}}(G \circ X_t/\alpha_t, \sigma_t^2/\alpha_t^2)\\
       =&\alpha_t \frac{\sum_{n} (G \circ X^{(n)}+G \circ \nabla \log p_d(X^{(n)}) ) p_d(G \circ X^{(n)}) }{\sum_n p_d(G \circ X^{(n)}) }- G\circ X_t\\
       =& G\circ \alpha_t \frac{\sum_{n} ( X^{(n)}+\nabla \log p_d(X^{(n)}) ) p_d(X^{(n)}) }{\sum_n p_d(X^{(n)}) }- G\circ X_t\\
       =& G\circ \left(\alpha_t \frac{\sum_{n} ( X^{(n)}+\nabla \log p_d(X^{(n)}) ) p_d(X^{(n)}) }{\sum_n p_d(X^{(n)}) }-  X_t\right), \quad X^{(1:N)} \sim \bar{\mathcal{N}}(X_t/\alpha_t, \sigma_t^2/\alpha_t^2).
   \end{align}
   The last line follows since $G \circ  X^{(1:N)} \sim \bar{\mathcal{N}}(G \circ X_t/\alpha_t, \sigma_t^2/\alpha_t^2)$ is equivalent to $=X^{(1:N)} \sim \bar{\mathcal{N}}(X_t/\alpha_t, \sigma_t^2/\alpha_t^2)$.
\end{proof}
\subsubsection{Sampling Importance Resampling}\label{sec:sir}
Instead of estimating the integral by IS, we can also perform Sampling Importance Resampling (SIR) using the importance weight.
Specifically, we can draw one sample $X^*$ from the Categorical distribution according to the IS weights:
\begin{align}\label{eq:SIR}
&X^* = X^{(n^*)},\nonumber\\
  \text{where }  &n^* \sim \text{Cat}\left( 
   \frac{p_d(X^{(1)})}{\sum_{n} p_d(X^{(n)})}, \frac{p_d(X^{(2)})}{\sum_{n} p_d(X^{(n)})}, \cdots, \frac{p_d(X^{(N)})}{\sum_{n} p_d(X^{(n)})}
    \right), \\ \text{and }&X^{(1:N)} \sim \bar{\mathcal{N}}(X_t/\alpha_t, \sigma_t^2/\alpha_t^2).\nonumber
\end{align}
The sample obtained by SIR is $G$-equivariant to $X_t$.
\begin{proof}
If we apply $G$ to $X_t$, SIR becomes
\begin{align}
&{X^*}' = G \circ{X^{(n^*)}},\nonumber\\
    \text{where }&n^* \sim \text{Cat}\left( 
   \frac{p_d(G\circ X^{(1)})}{\sum_{n} p_d(G\circ X^{(n)})}, \frac{p_d(G\circ X^{(2)})}{\sum_{n} p_d(G\circ X^{(n)})}, \cdots, \frac{p_d(G\circ X^{(N)})}{\sum_{n} p_d(G\circ X^{(n)})}
    \right), \\\text{and }&  G \circ {X^{(1:N)}} \sim \bar{\mathcal{N}}(G \circ X_t/\alpha_t, \sigma_t^2/\alpha_t^2).\nonumber
\end{align}
Since the target density is $G$-invariant and $G \circ {X^{(1:N)}} \sim \bar{\mathcal{N}}(G \circ X_t/\alpha_t, \sigma_t^2/\alpha_t^2)$ is equivalent to ${X^{(1:N)}} \sim \bar{\mathcal{N}}( X_t/\alpha_t, \sigma_t^2/\alpha_t^2)$, we have
\begin{align}\label{eq:SIR_G}
&{X^*}' = G \circ{X^{(n^*)}},\nonumber\\
    \text{where }&n^* \sim \text{Cat}\left( 
   \frac{p_d( X^{(1)})}{\sum_{n} p_d( X^{(n)})}, \frac{p_d( X^{(2)})}{\sum_{n} p_d( X^{(n)})}, \cdots, \frac{p_d( X^{(N)})}{\sum_{n} p_d( X^{(n)})}
    \right), \\\text{and }&  {X^{(1:N)}} \sim \bar{\mathcal{N}}( X_t/\alpha_t, \sigma_t^2/\alpha_t^2).\nonumber
\end{align}
Comparing \Cref{eq:SIR,eq:SIR_G}, we conclude ${X^*}' = G\circ X^*$, and hence the sample obtained by SIR is $G$-equivariant to $X_t$.
\end{proof}
Additionally, the score at the sample obtained by SIR is also  $G$-equivariant to $X_t$.
\subsubsection{Hamiltonian Monte Carlo}\label{sec:hmc_se3}
We now look at more complicated cases, where we run Hamiltonian Monte Carlo (HMC) or Langevin Dynamics (LG, including ULA and MALA) to obtain samples from $p_d(X|X_t)$.
We start with HMC in this section and look at LG in the next section.
Our conclusion will require the following two assumptions:
\begin{itemize}
    \item Assumption 1: the initial guess in HMC is zero-centered and $G$-equivariant w.r.t. to $X_t$.
This is a reasonable assumption, as the initial guess can simply be a sample from $\bar{\mathcal{N}}( X_t/\alpha_t, \sigma_t^2/\alpha_t^2)$, or can be a sample from Sampling Importance Resampling.
\item Assumption 2: the momentum variable in HMC follows zero-centered and $G$-invariant Gaussian, which also most holds true since the most common choice is standard or isotropic Gaussian. 
\end{itemize}
Under these assumptions, the samples obtained by HMC are zero-centered and $G$-equivariant to $X_t$.
In the following, we prove that this holds for the first step.
The other steps can be simply proved by viewing the sample from the previous guess as the initial guess.
\begin{proof}
   \textbf{(1) Zero-centered.}
We first note that the gradient of $\log p_d(X|X_t)$ is zero-centered if both $X$ and $X_t$ are zero-centered.
This is easy to check:
\begin{align}\label{eq:zero_centered}
    \nabla_X \log p_d(X|X_t) =  \underbrace{\nabla_X \log p_d(X)}_{\text{zero-centered by \Cref{prop:translation}}} +  \underbrace{\nabla_X \log \bar{\mathcal{N}}(X_t|X) }_{\propto X-X_t}.
\end{align} 
We now look at the HMC transition kernel (with frog-leap):
\begin{align}
\begin{split}
    &P \sim \bar{\mathcal{N}}(0, m),\\
    &P_{t/2} \leftarrow P + \underbrace{\frac{t}{2} \nabla_X \log p_d(X|X_t)}_{\text{zero-centered by \Cref{eq:zero_centered}}},  \\
    &X' \leftarrow X + t P_{t/2},\\
    &P' \leftarrow P_{t/2}+ \underbrace{\frac{t}{2} \nabla_{X'} \log p_d(X'|X_t)}_{\text{zero-centered by \Cref{eq:zero_centered}}}.
\end{split}
\end{align}
This proposed $X'$ will be accepted according to the Metropolis-Hastings criterion:
\begin{align}
    \alpha = \min\left\{1, \frac{\bar{\mathcal{N}}(P'|0, m)p_d(X'|X_t)}{\bar{\mathcal{N}}(P|0, m) p_d(X|X_t)}\right\}= \min\left\{1, \frac{\bar{\mathcal{N}}(P'|0, m)p_d(X')\bar{\mathcal{N}}(X_t|X')}{\bar{\mathcal{N}}(P|0, m) p_d(X)\bar{\mathcal{N}}(X_t|X)}\right\}.
\end{align}
Since $P$ is zero-centered by Assumption 2, all operations will maintain the zero-centered property.
\par
\textbf{(2) $G$-equivariant. }
It is easy to check $\alpha$ is $G$-invariant.
We, therefore, only focus on proving the frog-leap step is $G$-equivariance.
We simplify the frog-leap steps involving $X$ as:
\begin{align}
    X' \leftarrow X + t
P + \frac{t}{2} \nabla \log p_d( X|X_t), \text{ where }  P  \sim \bar{\mathcal{N}}( 0, m).
\end{align}
Notice that, after applying $G$ to $X$ and hence $X_t$ (by Assumption 1 that the initial guess of $X$ is $G$-equivariant to $X_t$), the frog-leap steps involving $X$ become:
\begin{align}
    X'' &\leftarrow G \circ X + t \left(
  G \circ  P + \frac{t}{2} \underbrace{\nabla \log p_d(G \circ X|G \circ X_t)}_{=G \circ \nabla \log p_d( X| X_t)}\right), \text{ where } G \circ P \sim \bar{\mathcal{N}}(0, m)\\
  &=G \circ  \left(X + t
P + \frac{t}{2} \nabla \log p_d( X|X_t)\right), \text{where }  P \sim \bar{\mathcal{N}}(G^{-1} \circ 0, M) = \bar{\mathcal{N}}( 0, m)\\
&=G \circ X'.
\end{align}
This completes the proof.
\end{proof}

\subsubsection{Langevin Dynamics}\label{sec:ld_se3}
We now look at LG (i.e., ULA and MALA).
By the same argument as HMC, the Metropolis-Hastings acceptance is $G$-invariant and will not influence our conclusions. 
We, therefore, simply look at a single step in ULA but note that the same holds for MALA as well.
We also take the two assumptions made in HMC with a small modification:
\begin{itemize}
    \item Assumption 1: the initial guess in LG is zero-centered and $G$-equivariant w.r.t. to $X_t$.
\item Assumption 2: the Brownian motion we take in LG is izero-centered. 
\end{itemize}
Under these assumptions, the samples obtained by LG are zero-centered and $G$-equivariant to $X_t$.
\begin{proof}
    \textbf{(1) Zero-centeredness. }This is trivial by the LG updating formula: 
     \begin{align}
       X' \leftarrow X +  \gamma  \underbrace{\nabla_X \log p_d(X|X_t)}_{\text{zero-centered by \Cref{eq:zero_centered}}} 
 + \sqrt{2\gamma} \mathcal{E},
    \end{align}
where $\mathcal{E}$ is a matrix of standard Gaussian noise in the zero-centered subspace, i.e., $\mathcal{E}\sim \bar{\mathcal{N}}(0, 1)$. 
\par
\textbf{(2) $G$-equivariance. }
Applying $G$ to both $X_t$ and $X$, and noticing that the Gaussian distribution over $\mathcal{E}$ is $G$-invariant, we obtain the new LG updating formula: 
\begin{align}
       X'' &\leftarrow G \circ X +  \gamma  \underbrace{\nabla \log p_d(G \circ X|G \circ X_t)}_{=G \circ \nabla \log p_d( X| X_t)}  + \sqrt{2\gamma} \mathcal{E}, \quad \mathcal{E}\sim \bar{\mathcal{N}}(0, 1)\\
       &= G \circ \left( X +  \gamma  \nabla \log p_d( X|X_t)  + \sqrt{2\gamma} \mathcal{E}\right), \quad \mathcal{E}\sim \bar{\mathcal{N}}(G^{-1} \circ 0, 1) = \bar{\mathcal{N}}(0, 1)\\
       &=G \circ X'.
\end{align}
This completes the proof.
\end{proof}

\subsubsection{Annealed Importance Sampling}
Another possible choice for the score estimator is Annealed Importance Sampling \citep[AIS, ][]{neal2001annealed} or AIS followed by importance resampling.
Recall that in IS, we draw samples from the proposal $\bar{\mathcal{N}}(X| X_t/\alpha_t, \sigma_t^2/\alpha_t^2)$, and the target as given by $p_d(X) \bar{\mathcal{N}}(X_t|\alpha_t X, \sigma_t^2)$.
For AIS, we introduce a sequence of intermediate distributions that interpolate between the proposal and the target: 
\begin{align}
    \pi_{(k)}(X) &\propto \left(\underbrace{\bar{\mathcal{N}}(X| X_t/\alpha_t, \sigma_t^2/\alpha_t^2)}_{\text{proposal}}\right)^{1-\beta_k} \left(\underbrace{p_d(X) \bar{\mathcal{N}}(X_t|\alpha_t X, \sigma_t^2)}_{\text{target}}\right)^{\beta_k} \\&\propto \left(p_d(X)\right)^{\beta_k} \bar{\mathcal{N}}(X_t|\alpha_t X, \sigma_t^2).
\end{align}
where $\beta_0=0$ and $\beta_K=1$.
AIS follows an iterative process:
The AIS algorithm proceeds iteratively as follows: \begin{itemize} 
\item Draw $X_{(0)} \sim \pi_{(0)}$; 
\item For $k = 1, 2, \dots, K-1$, run a MCMC using a transition kernel $T(X_{(k)}|X_{(k-1)})$, with $\pi_{(k)}$ as the stationary distribution, to obtain $X_{(k)}$. 
\end{itemize}
In the end, we calculate the IS weight in the joint space defined over $X_{(1:K-1)}$.
This will yield the AIS weight:
\begin{align}
    w_\text{AIS}(X_{(1:K-1)}) &= \frac{\pi_{(1)}(X_{(0)})}{\pi_{(0)}(X_{(0)})}\frac{\pi_{(2)}(X_{(1)})}{\pi_{(1)}(X_{(1)})}\cdots \frac{\pi_{(K)}(X_{(K-1)})}{\pi_{(K-1)}(X_{(K-1)})}\\
    &= \prod_{k=1}^K \left(p_d(X_{(k-1)})\right)^{\beta_k - \beta_{k-1}}.
\end{align}
Therefore, MSI can be estimated by
\begin{align}
   & \int (\alpha_t(X{+}\nabla \log p_d(X)){-}X_t) p_d(X|X_t) dx\\
   \approx& \alpha_t \frac{\sum_{n} (X^{(n)}+\nabla \log p_d(X^{(n)}) ) w_\text{AIS}(X_{(1:K-1)}^{(n)}) }{\sum_n  w_\text{AIS}(X_{(1:K-1)}^{(n)}) }- X_t, \quad X_{(1:K-1)}^{(1:N)} \sim \text{AIS}.
\end{align}
We use superscripts to represent the sample index and subscripts to denote the intermediate step index in AIS.
\par
Assume we use HMC or LG with the assumptions discussed in \Cref{sec:hmc_se3,sec:ld_se3} as the transition kernel in AIS.
The sequence of samples we obtain ($X_{(1:K-1)}^{(1:N)}$) will be $G$-equivariant w.r.t. $X_t$.
Additionally, notice that $\nabla \log p_d(X^{(n)}))$ is $G$-invariance to $X^{(n)}$ and $w_\text{AIS}$ is $G$-invariance as $p_d$ is $G$-invariance.
We can conclude that the AIS estimator for MSI is $G$-equivariant w.r.t. $X_t$.
\par
Following the same argument as in \Cref{sec:sir}, if we perform AIS followed by a resampling, the result will also be  $G$-equivariant w.r.t. $X_t$.

\section{EXPERIMENT DETAILS}\label{appendix:experiment_setup}

\subsection{Mixture of 40 Gaussians (MoG-40)}

We employ the mixture of 40 Gaussians (MoG-40) target distribution in 2D proposed in \citet{midgley2023flow} to visually examine the mass-covering property of different models. 
We train all methods for 2.5h, which allows all of them to converge.

For our approach, we choose the total number of diffusion steps $T=30$. 
We use a variance-preserving (VP) scheme \citep{ho2020denoising} and a linear schedule with $\beta_t$ ranging from $10^{-4}$ to $0.7$.
We choose the weighting function to be $w(t)=1/\alpha_t$. 
For the score network $s_\phi(x_t)$, we use a 5-layer MLP with hidden dimension $400$ and SiLU activation.
In each inner loop, we use Adam to train the score network $s_\phi(x_t)$ for 50 iterations using DSM with learning rate $10^{-4}$ and batch size $1,024$.
For the neural sampler $g_{\theta}(z)$, we use a 5-layer MLP with latent dimension $2$, hidden dimension $400$ and SiLU activation.
We use Adam to train the neural sampler $g_{\theta}(z)$ using MSI with learning rate $10^{-3}$, batch size $1,024$, and gradient norm clip $10.0$.
Regarding posterior sampling for \Cref{eq:posterior}, we use AIS with 10 importance samples, 15 AIS steps.
For each AIS step, we use HMC transition kernel with 1 frog-leap step and step size $1.0$. 
We resample one of those 10 AIS samples according to the AIS weights and use that sample as initialization for 5 steps of MALA with step size $10^{-2}$.
In the end, we estimate the MSI using this single sample (i.e., we perform Monte Carlo estimation with one sample).

For R-KL-based approaches, we align their experiment setups with that for our approach. 
Specifically, we use a 5-layer MLP with hidden dimension $400$ and SiLU activation for the score network $s_\phi(x_t)$.
In each inner loop, we use Adam to train the score network $s_\phi(x_t)$ for 50 iterations using DSM with learning rate $10^{-4}$ and batch size $1,024$.
For the neural sampler $g_{\theta}(z)$, we use a 5-layer MLP with latent dimension $2$, hidden dimension $400$ and SiLU activation.
We use Adam to train the neural sampler $g_{\theta}(z)$ with learning rate $10^{-3}$, batch size $1,024$, and gradient norm clip $10.0$.

For FAB \citep{midgley2023flow} and iDEM \citep{akhound2024iterated}, we use exactly the same setups as described in the respective papers. Note that both of them use replay buffers. In addition, all methods except iDEM work under the original scale $[-50, 50]$ of the target. iDEM normalizes the target to the range $[-1, 1]$, which may simplify the task.

\subsection{Many-Well-32 (MW-32, Internal Coordinate)}

We employ the Many-Well target distribution in 32D proposed in \citet{midgley2023flow} to examine the mass-covering property of different models, as this target contains $2^{16}$ modes. We train all models until convergence. Training and sampling time for each model can be found in \Cref{tab:time} in the main text.

For our approach, we choose the total number of diffusion steps  $T=30$. 
We use a variance-preserving (VP) scheme \citep{ho2020denoising} and a linear schedule with $\beta_t$ ranging from $10^{-4}$ to $0.15$.
We choose the weighting function to be $w(t)=1/\alpha_t$. 
For the score network $s_\phi(x_t)$, we use a 5-layer MLP with hidden dimension $400$ and SiLU activation.
In each inner loop, we use Adam to train the score network $s_\phi(x_t)$ for 50 iterations using DSM with learning rate $10^{-4}$ and batch size $1,024$.
For the neural sampler $g_{\theta}(z)$, we use a 5-layer MLP with latent dimension $32$, hidden dimension $400$ and SiLU activation.
We use Adam to train the neural sampler $g_{\theta}(z)$ using MSI with learning rate $10^{-3}$, batch size $1,024$, and gradient norm clip $10.0$.
Regarding posterior sampling for \Cref{eq:posterior}, we use AIS with 10 importance samples, 15 AIS steps.
For each AIS step, we use HMC transition kernel with 1 frog-leap step and step size $0.3$. 
We resample one of those 10 AIS samples according to the AIS weights and use that sample as initialization for 5 steps of MALA with step size $5\times 10^{-2}$.
In the end, we estimate the MSI using this single sample.

For R-KL-based approaches, we align their experiment setups with that for our approach. 
Specifically, we use a 5-layer MLP with hidden dimension $400$ and SiLU activation for the score network $s_\phi(x_t)$.
In each inner loop, we use Adam to train the score network $s_\phi(x_t)$ for 50 iterations using DSM with learning rate $10^{-4}$ and batch size $1,024$.
For the neural sampler $g_{\theta}(z)$, we use a 5-layer MLP with latent dimension $32$, hidden dimension $400$ and SiLU activation.
We use Adam to train the neural sampler $g_{\theta}(z)$ with learning rate $10^{-3}$, batch size $1,024$, and gradient norm clip $10.0$.

For FAB, we use exactly the same setup as described in \citet{midgley2023flow}. 
For iDEM, we use the same setup as that for experiments in internal coordinates as described in \citep{akhound2024iterated} but change the maximum score norm clip threshold to $1,000$, increase the number of MC samples to $1,000$, and reduce $\sigma_{max}$ in the noise schedule to $1.0$. Note that both of FAB and iDEM use replay buffers.

\subsection{Double-Well-4 (DW-4, Cartesian Coordinate)}

We employ the Double-Well target with 4 particles in 2D.
This target was originally introduced by \citet{kohler2020equivariant} and also used in \citet{midgley2024se,akhound2024iterated} to evaluate model performance on invariant targets.
We train all models until convergence. 
Training and sampling time for each model can be found in \Cref{tab:time} in the main text.
\par
For our approach, we choose the total number of diffusion steps $T=30$. 
We use a variance-preserving (VP) scheme \citep{ho2020denoising} and a linear schedule with $\beta_t$ ranging from $10^{-6}$ to $0.05$.
We found that using a constant weighting function, 
$w(t)=1$, is beneficial for handling these complex targets.
We use EGNN following \citet{hoogeboom2022equivariant} for both the score network $s_\phi$ and the neural sampler $g_\theta$.
The neural sampler has  8 layers with a hidden dimension of 144 and ReLU activation.
The score network has 4 layers with the same width, and it is additionally conditioned on $t$.
We use Adam to train the score network $s_\phi$ for 100 iterations using DSM with learning rate $10^{-4}$ and batch size $1024$.
We use Adam to train the neural sampler $g_{\theta}(z)$ using MSI with learning rate $5\times 10^{-4}$, batch size $1024$, and gradient norm clip $10.0$.
Regarding posterior sampling for \Cref{eq:posterior}, we use AIS with 20 importance samples, 10 AIS steps.
For each AIS step, we use 1-step MALA transition kernel with step size $0.01$. 
We resample one of those 20 AIS samples according to the AIS weights and use that sample as initialization for 50 steps of MALA.
We dynamically adjust the MALA step size to maintain an acceptance rate between 0.5 and 0.6. 
Specifically, we increase the step size by a factor of 1.5 when the acceptance rate exceeds 0.6 and decrease it by a factor of 1.5 when the acceptance rate drops below 0.5.
In the end, we estimate the MSI using this single sample.
\par
Additionally, we employ early stopping during training.
Specifically, we generate 2,000 samples using the neural sampler, which serve as the \emph{predictions}. 
These samples are then used as the initialization for 50 MALA steps, targeting the target energy. 
The 2,000 samples obtained after MALA are treated as the \emph{validation} set. 
We evaluate the energy of both the predictions and the validation set, then calculate the total variation (TV) distances between their energy histograms. 
We save the model with the lowest TVD.
This criterion can be interpreted as asking: 
how much improvement can be achieved by applying a small number of Langevin dynamics to the model samples? 
The less improvement we can achieve, the better the model is.
\par
For FAB \citep{midgley2023flow} and iDEM \citep{akhound2024iterated}, we use exactly the same setups as described in the respective papers. Note that both of them use replay buffers.

\subsection{Lennard-Johns-13 (LJ-13, Cartesian Coordinate)}
We employ the Lennard-Jones target with 13 particles in 3D, as introduced by \citet{kohler2020equivariant} and later used in \citet{midgley2024se, akhound2024iterated} to assess model performance on invariant targets.
This target is more complex than DW-4 in the sense that its energy landscape includes prohibitive regions that can destabilize training. 
However, DW-4 poses its own challenges, as it has two modes, and balancing these modes can be more difficult than handling the Lennard-Jones target.
We therefore evaluate our approach and baselines on both to evaluate its behavior comprehensively.
We train all models until convergence. 
Training and sampling time for each model can be found in \Cref{tab:time} in the main text.
\par
For our approach, we choose the total number of diffusion steps $T=30$. 
We use a variance-preserving (VP) scheme \citep{ho2020denoising} and a linear schedule with $\beta_t$ ranging from $10^{-6}$ to $0.05$.
We also use a constant weighting function, 
$w(t)=1$.
Both the score network $s_\phi$ and the neural sampler network $g_\theta$ share the same 8-layer architecture with a hidden dimension of 192 and ReLU activation. 
We use Adam to train the score network $s_\phi$ for 100 iterations using DSM with learning rate $10^{-4}$ and batch size $256$.
We use Adam to train the neural sampler $g_{\theta}$ using MSI with learning rate $5\times 10^{-4}$, batch size $256$, and gradient norm clip $10.0$.
Regarding posterior sampling for \Cref{eq:posterior}, we use IS with 500 importance samples.
We then resample one of those 20 AIS samples according to the AIS weights and use that sample as initialization for 1,000 steps of MALA.
We also dynamically adjust the MALA step size to maintain an acceptance rate between 0.5 and 0.6. 
Specifically, we increase the step size by a factor of 1.5 when the acceptance rate exceeds 0.6 and decrease it by a factor of 1.5 when the acceptance rate drops below 0.5.
Unlike previous tasks, we found that using only the last sample from MALA sometimes leads to suboptimal performance. 
To improve stability, we track the samples and their gradients from the last 500 steps of MALA, and estimate the MSI using these 500 samples. 
It's important to note that since the samples and their scores are already computed during MALA, using more samples in the Monte Carlo estimator does not incur any additional computational cost.
We found smoothing the LJ target following \citet{moore2024computing} can help to stabilize the training.
However, our approach works well even without this smoothing.
Additionally, we employ early stopping in the same way as in DW-4.
\par
For FAB \citep{midgley2023flow} and iDEM \citep{akhound2024iterated}, we use exactly the same setups as described in the respective papers. Note that both of them use replay buffers.

\subsection{Summary and Guidance for Hyperparameter Tuning in DiKL}
Below, we summarize some key hyperparameters for our method:
\par
(1) $\beta_t$ should not be too large, as this can lead to inaccurate posterior sampling and worse performance.
The neural sampling tends to favor the mean of the global mass, which can be seen in \Cref{fig:vis_why_drKL_works} in the main text: as the noise level increases, the model is biased towards the mean.
\par
(2) For DW and LJ tasks, it is important to use a large EGNN for the neural sampler. 
Unlike diffusion models with EGNNs \citep{akhound2024iterated, hoogeboom2022equivariant}, our approach involves learning a one-step sampling generator, requiring greater model capacity.
 We tested smaller networks, such as an EGNN with 6 layers and 128 hidden dimensions. 
 However, these yielded worse performance compared to the larger architecture used in our experiments.
On the other hand, the score network does not need to have the same capacity.
 For example, for DW4, we found a shallower one that suffices for optimal performance.
 Additionally, interestingly, we found using ReLU yields better performance than SiLU in EGNNs.
\par
(3) The weight function $w(t)$  also requires careful consideration.
While other weighting functions commonly used in diffusion models include $\sigma_t^2/\alpha_t$ or $\sigma_t^2/\alpha_t^2$, we found using $1/\alpha_t$ or uniform weighting is more stable in our approach.
An empirical guideline for choosing between these is as follows: 
for more complex targets like DW and LJ, a uniform weighting function can better at encouraging exploitation. 
On the other hand, for highly multi-modal targets, using 
$1/\alpha_t$ can accelerate exploration.
\par
(4) The batch size cannot be too small. Empirically, we found that training could be unstable if  a small batch size is used. In general, we recommend using a large batch size like 1,024 if it fits into the GPU memory.
\par
On the other hand, perhaps surprisingly, it is not crucial to have a perfect posterior sampler.
Our method essentially works in a bootstrapping manner: the posterior samples improve the model, and a better model in turn brings the posterior samples closer to the true target. 
Having said that, accurate posterior sampling may improve the convergence rate of the model.

\section{ADDITIONAL EXPERIMENTS AND RESULTS}
\subsection{Results with Different Random Seeds}
Our training process can be subject to randomness. 
Therefore, in this appendix, we provide results obtained with different random seeds in \Cref{fig:seeds}.
 As we can see, different seeds can have different converge rates and also slightly influence the final performance (for example, the green line in LJ-13 achieves lower TVD than others).
    However, we observe that 
    (1) a longer training process consistently yields better and more stable results, regardless of the seed; 
    (2) at the end of training, performance may still exhibit slight fluctuations, making the early stopping we employed necessary;
    and (3) while different seeds introduce minor variations in final performance, these discrepancies are negligible and smaller than the observed fluctuations.

\begin{figure}
    \centering
    \begin{subfigure}{0.49\textwidth}
         \includegraphics[width=\linewidth]{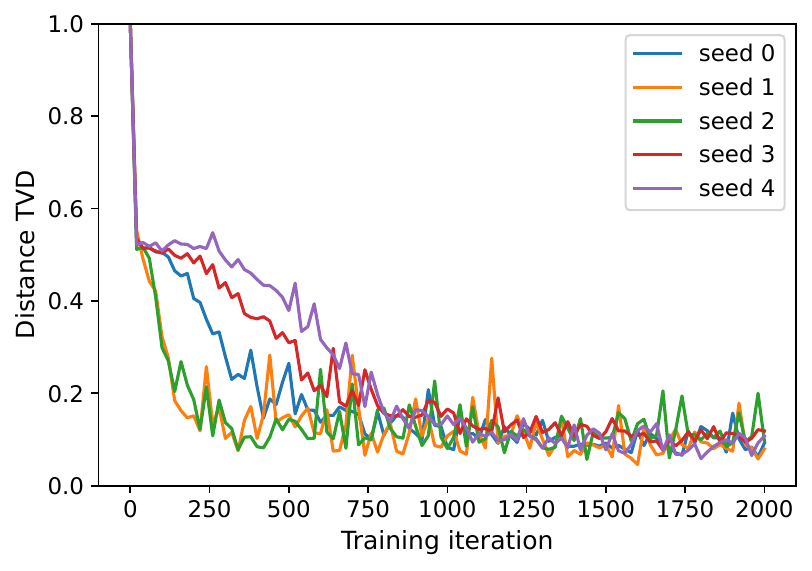}
    \caption{DW-4}
    \end{subfigure}
    \begin{subfigure}{0.49\textwidth}
         \includegraphics[width=\linewidth]{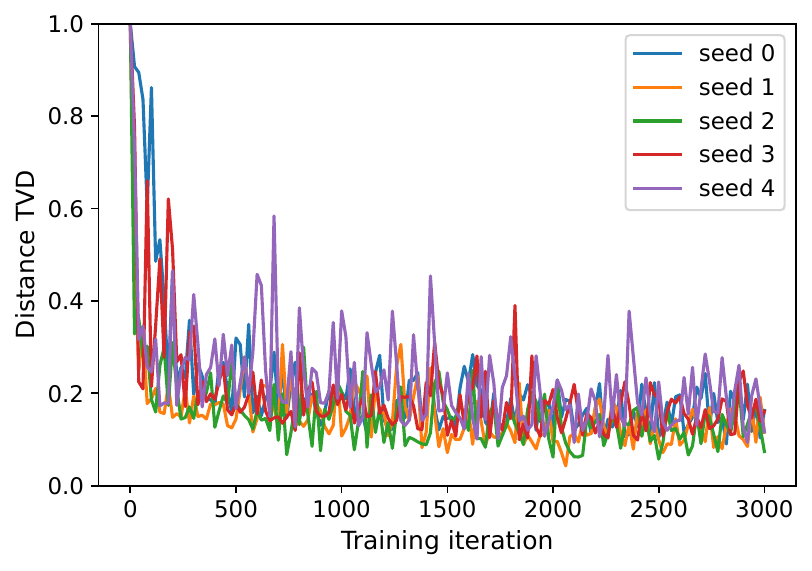}
    \caption{LJ-13}
    \end{subfigure}
    \caption{Distance TVD along the training process with different seeds. 
   }\label{fig:seeds}
\end{figure}

\subsection{Additionally Illustrations of \Cref{fig:vis_why_drKL_works}}
As suggested by anonymous reviewers, in \Cref{fig:vis_why_drKL_works2}, we provide an additional visualization of \Cref{fig:vis_why_drKL_works}, illustrating DiKL alongside KL divergence at different noise levels. This visualization highlights how DiKL balances model-seeking and mode-covering.
This visualization also addresses a potential concern: KL divergence at high noise levels can make the loss landscape flat, potentially hindering optimization. 
We observe that this is indeed the case when the noise level is large (e.g., 5.0). 
However, by summing over all noise levels in DiKL, this issue is mitigated, resulting in a smoother yet non-flat landscape that is easier to optimize.

However, it is important to acknowledge the limitation of this visualization: we use a Gaussian model to fit a two-mode Gaussian Mixture Model (GMM). As a result, even at optimal performance, the model cannot perfectly fit the target distribution. Therefore, when interpreting this plot, the focus should be on whether the model tends to converge to a single local mode or achieves a global coverage of both modes.

\begin{figure}[H]
    \centering
    \includegraphics[width=\linewidth]{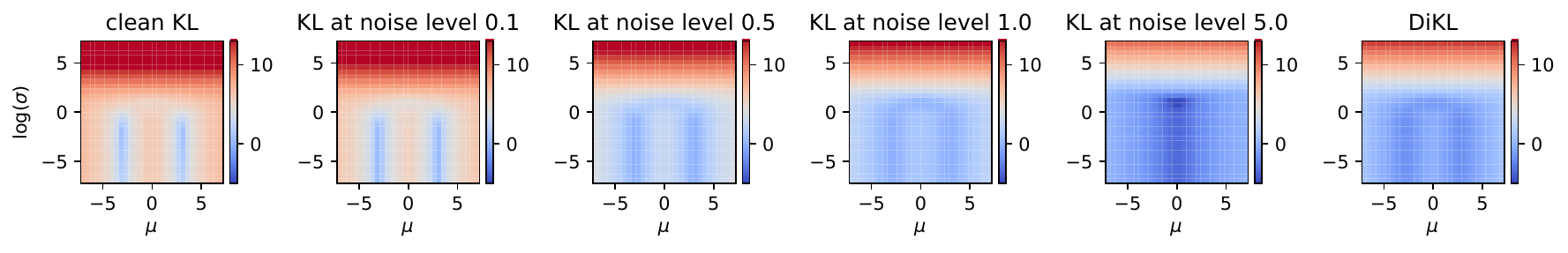}
    \caption{Heatmap of (log scale) KL divergence at different noise levels between a Gaussian model (with mean parameter $\mu$ and standard deviation parameter $\sigma$) and a two-mode MoG target in 1D. In  the last plot, we show the landscape of DiKL, where we sum over the KL divergence on all noise levels.
  }
    \label{fig:vis_why_drKL_works2}
\end{figure}

\subsection{Results on LJ-55}\label{appendix:lj55}

As we discussed in the main text, although our one-step generator $g_{\theta}(z)$ offers significantly faster sampling compared to diffusion-based samplers, it has limited model flexibility relative to multi-step diffusion models such as iDEM. This limitation becomes more pronounced when the target distribution is complex and has a larger Lipschitz constant.
To have a comprehensive evaluation of the our approach's weakness, we present the results of DiKL on the LJ-55 potential well in \Cref{fig:lj55}.
A promising direction for future work is to combine DiKL with multi-step samplers. 
One possible approach is to introduce several interpolants between the prior and the target distribution and train DiKL sequentially to transport between adjacent interpolants. 
We leave this for future exploration.

\begin{figure}[H]
    \centering
    \includegraphics[width=0.3\linewidth]{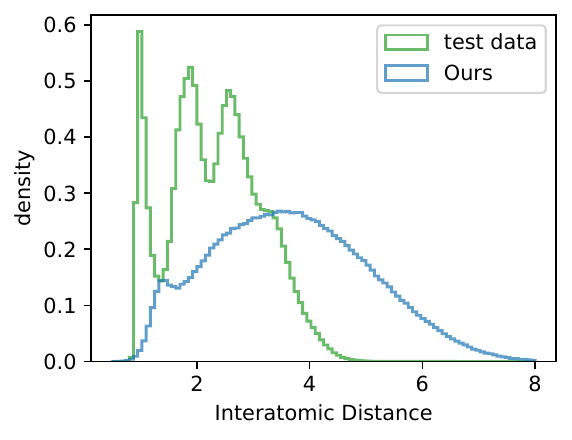}
    \caption{Performance on complex system LJ-55. }
    \label{fig:lj55}
\end{figure}

\section{ASSETS AND LICENSES}
We use the following codebases for baselines and benchmarks in our experiments:
\begin{itemize}
    \item FAB (MIT license): PyTorch implementation for MoG-40 and MW-32 (\url{https://github.com/lollcat/fab-torch})
    ; JAX implementation for DW-4 and LJ-13 (\url{https://github.com/lollcat/se3-augmented-coupling-flows}).
    \item iDEM (MIT license): PyTorch implementation for all experiments (\url{https://github.com/jarridrb/DEM}).
    \item DW-4 and LJ-13 target energy functions (MIT license): bgflow (\url{https://github.com/noegroup/bgflow}).
\end{itemize}

\end{document}